%% file: main.tex
\documentclass[11pt]{article}

\input{preamble}

\title{
Dimension-Free Rank Lifting from Random Hyperplane Arrangements}
\author{Luca Becchetti\thanks{Sapienza University of Rome, Italy,
  \textsf{becchetti@diag.uniroma1.it}.}%
  \and
  Matteo Russo\thanks{EPFL, Switzerland,
  \textsf{\{matteo.russo, ruben.skorupinski\}@epfl.ch}.
  }%
  \and
  Ruben Skorupinski\footnotemark[2]%
}
\date{}

\begin{document}

\maketitle

\makeatletter
\renewcommand{\thefootnote}{}
\footnotetext{The authors are listed in alphabetical order.}
\addtocounter{footnote}{-1}
\makeatother

\begin{abstract}
We study the width required for a randomly initialized hidden layer of a neural network to achieve \emph{rank lifting}. Namely, given a dataset $X \in \mathbb R^{m \times d}$ of \(m\), $d$-dimensional input vectors separated by an angle of at least \(\theta\), we consider the random feature matrix \(\sigma(XR)\), where \(R\) is standard Gaussian. For positively homogeneous nonpolynomial activations, which include sign, Heaviside, ReLU, and ReLU powers among others, we prove that
\[
    n \gtrsim \frac{1}{\theta}\max\left\{m,\log\left(\frac1\delta\right)\right\}
\]
neurons suffice for \(\sigma(XR)\) to have full row rank $m$ with probability at least \(1-\delta\). This dimension-free bound exponentially improves the previous general-dimensional guarantee for sign features \cite{DragoBCMSSB26} and is essentially tight. The proof shows that one random feature column escapes every proper subspace of \(\mathbb R^m\) with probability \(\Omega(\theta)\), using a coupling of nearby Gaussian directions and a local crossing of the induced hyperplane arrangement. We also study \emph{stable rank lifting}, where the goal is to establish a quantitative analogue of exact rank lifting, i.e., a lower bound on the smallest eigenvalue of the empirical feature Gram matrix in high-probability. 
Our analysis unifies and generalizes stable rank guarantees for all $q$-homogeneous non-polynomial activations following prior work in \cite{PanigrahiShettyGoyal20, Song26}. In particular, we combine a diagonally dominant Taylor tail of the population kernel with truncation and matrix concentration, to show that for positively homogeneous nonpolynomial activations, stable rank lifting is achieved at width
\[
    n \gtrsim C^q \frac{m}{\theta^{2q+1}} \log^{2q+\frac{1}{2}}\left(\frac{m}{\theta}\right) \log\left(\frac{m}{\delta}\right),
\]
where \(q\) is the degree of the activation and $C > 0$ is some universal constant. 
\end{abstract}

\thispagestyle{empty}

\input{sections/introduction}

\input{sections/preliminaries}

\input{sections/exact_rank}

\input{sections/stable_rank}

\section*{Acknowledgments}

The work of Ruben Skorupinski is supported by project 10000183 of the Swiss National Science Foundation (SNSF). We used ChatGPT (GPT-5.6 Sol and GPT-6) to assist with mathematical arguments. Specifically, for exact rank lifting, it helped generalize the non-vanishing argument from the sign activation to positively homogeneous nonpolynomial activations. For stable rank lifting, it assisted with calculations for deriving the Taylor coefficients of the Gaussian correlation kernel, making the constants in the eigenvalue and concentration bounds explicit. The authors take full responsibility for all the content of this work.

\bibliography{references}
\bibliographystyle{alpha}

\input{sections/appendix}

\end{document}

%% file: preamble.tex
\usepackage[T1]{fontenc}
\usepackage{lmodern,microtype}
\usepackage{inconsolata}
\usepackage{libertine}
\usepackage[margin=1in]{geometry}
\usepackage[colorlinks=true,linkcolor=blue,citecolor=blue,urlcolor=blue]{hyperref}
\usepackage{amsmath,amssymb,amsthm,mathtools}
\usepackage{bbm}
\usepackage{enumitem}
\usepackage{nicefrac}
\usepackage[capitalise,nameinlink]{cleveref}
\usepackage{tcolorbox}
\usepackage{booktabs}
\usepackage{thm-restate}

\newtheorem{theorem}{Theorem}[section]
\newtheorem{proposition}[theorem]{Proposition}
\newtheorem{lemma}[theorem]{Lemma}

\newtheorem{assumption}[theorem]{Assumption}
\newtheorem{remark}[theorem]{Remark}

\newtheorem{observation}[theorem]{Observation}
\newtheorem{question}[theorem]{Question}

\newcommand{\R}{\mathbb R}

\newcommand{\E}{\mathbb E}
\newcommand{\Pb}{\mathbb P}

\newcommand{\sgn}{\operatorname{sgn}}
\newcommand{\rank}{\operatorname{rank}}
\newcommand{\Span}{\operatorname{span}}

\newcommand{\diag}{\operatorname{diag}}

\def\bw{{\mathbf w}}

%% file: sections/introduction.tex
\section{Introduction}\label{sec:intro}

Randomly initialized and frozen hidden layers have been proposed as a tool to improve the expressive power of neural networks before a single
training step is taken. They appear in random-basis networks \cite{IgelnikPao95}, extreme learning machines \cite{HuangZhuSiew06}, reservoir computing \cite{Jaeger01,MaassNatschlaegerMarkram02,LukoseviciusJaeger09}, and random-feature methods for kernel learning \cite{RahimiRecht07,RahimiRecht08}. In this context, full rank of the random feature matrix is a particularly clean certificate of representational capacity at initialization: it separates the question of whether the random representation is expressive enough from the subsequent question of how that
representation is optimized.

More precisely, consider a \emph{projectively separated} dataset, i.e., a
dataset \(X\) composed of nonzero data points \(x_1,\ldots,x_m\in\mathbb R^d\), whose
pairwise angles lie between \(\theta\) and \(\pi-\theta\), for
\(0<\theta\le\pi/2\). Let
\(g_1,\ldots,g_n\in\mathbb R^d\) be independent standard Gaussian hidden-layer
weights, grouped as the columns of \(R\in\mathbb R^{d\times n}\), and let
\(\sigma:\mathbb R\to\mathbb R\) be an activation function. The corresponding
two-layer random feature model is defined as
\[
    f_{\bw}(x)
    =
    \sum_{j=1}^n w_j\sigma(\langle g_j,x\rangle),
\]
where only the output weights
\(\bw=(w_1,\ldots,w_n)\) are trained. Then, given the dataset \(X\), the vector
of predictions on the training set is \(Z\bw\), where
\(Z_{ij}=\sigma(\langle x_i,g_j\rangle)\). We also let $u_i=x_i/\|x_i\|_2$, and denote by \(U\in\mathbb R^{m\times d}\) the matrix with rows \(u_i^\top\). 

The input matrix may have rank much smaller than \(m\), particularly when
\(m>d\), but a nonlinear random map can nevertheless produce a feature matrix
of full row rank. In particular,
\(\rank(Z)=m\) precisely when every target vector \(y\in\mathbb R^m\) can be
interpolated by training only the last layer. Following \cite{DragoBCMSSB26}, this phenomenon is known as
\emph{rank lifting}. Rank lifting provides a simple model of how
overparameterization creates representational capacity: a sufficiently wide
random hidden layer separates the training points algebraically, after which a
linear readout can memorize arbitrary labels
\cite{ZhangEtAl17,Daniely20,Vershynin20,MontanariZhong22}. The recent work
\cite{DragoBCMSSB26} initiated a quantitative study of this question for
Gaussian hidden weights and sign activations. It interpreted rank lifting
geometrically through the arrangement of the activation hyperplanes on the
unit sphere and derived high-probability width guarantees from estimates on
the regions of this arrangement. In general dimension, the resulting bound on
the required width grows exponentially with the number of data points and linearly in $\sqrt{d}$, with sharper linear and polynomial guarantees for \(d=2\)
and \(d=3\), respectively.

\subsection{Our Contributions and Techniques}
In this paper, we considerably expand \cite{DragoBCMSSB26}, by investigating two questions that address complementary forms of rank lifting: the first asks only whether the random feature matrix has full row rank, whereas the second asks whether it is quantitatively nonsingular. Our approach to both questions relies on the following dual perspective on the problem: Each hidden weight \(g_j\) defines a random hyperplane \(g_j^\perp\), and the
\(j\)-th column of \(Z\) records the response of the data to this hyperplane.
Dually, the data hyperplanes \(u_i^\perp\) partition weight space into chambers
sampled by the Gaussian columns of \(R\).
The formal setup and standing assumptions are collected in \Cref{sec:preliminaries}.

\begin{question}[Exact Rank Lifting]\label{q:exact_rank_lifting}
Our first question asks how large the hidden-layer width \(n\) must be to ensure
that
\[
    \Pb\bigl(\rank(\sigma(XR))=m\bigr)\ge1-\delta.
\]
Equivalently, how many random neurons suffice to interpolate arbitrary labels
on the \(m\) data points by training only the output layer?
\end{question}

Our first result establishes a dimension-free width bound (\Cref{thm:exact-rank-homogeneous}) for the general-dimensional question, replacing the exponential dependence on \(m\) by a linear one, removing any
dependence on the ambient dimension \(d\), and extending the analysis from sign
activations to the broad class of \emph{positively homogeneous nonpolynomial
activations}, which includes Heaviside, ReLU, nonlinear leaky ReLU, ReLU powers, and sign itself. 
\begin{theorem}[Thm. \ref{thm:exact-rank-homogeneous} informal]
There is a universal constant \(C>0\) such that
\[
    n
    \ge
    \frac{C}{\theta}
    \max\left\{
        m,\log\left(\frac{1}{\delta}\right)
    \right\}
    \quad\Longrightarrow\quad
    \Pb\bigl(\rank(\sigma(XR))=m\bigr)\ge1-\delta.
\]
\end{theorem}
This result applies simultaneously to sign, ReLU, nonlinear leaky ReLU, and all
ReLU powers. For sign features, it replaces the exponential dependence on
\(m\) in \cite{DragoBCMSSB26} by a linear one, without any dependence on the
ambient dimension \(d\).
The proof uses local facet crossings rather than global chamber-volume estimates. Moreover, the bound is optimal in the following sense: first, \(n\ge m\) is necessary for any \(m\times n\) matrix to
have full row rank, and, second, configurations of two sign inputs or three ReLU
inputs show that failure probability at most \(\delta\) may require
\(n=\Omega(\theta^{-1}\log\delta^{-1})\). Thus the dependence on \(m\) is
optimal for fixed \(\theta\), and the dependence on \(\theta\) and \(\delta\)
is optimal for a constant number of data points (see \Cref{prop:exact-rank-confidence-lower-bounds} for a proof).

\medskip
\emph{Exact rank lifting} is an algebraic property as it distinguishes a zero
singular value from a nonzero one, but it is insensitive to how small a
nonzero singular value may be. It is then natural to investigate the question
of when rank lifting is quantitatively stable, as measured by the smallest eigenvalue of the empirical feature Gram matrix,
a property we call \emph{stable rank lifting}. 
\begin{question}[Stable Rank Lifting]\label{q:stable_rank_lifting}
Our second question asks how large should \(n\) be so that
\[
    \Pb\left(
        \lambda_{\min}
        \bigl(
            \sigma(UR)\sigma(UR)^\top
        \bigr)
        \ge
        cn\kappa_\sigma
    \right)
    \ge
    1-\delta,
\]
for a universal constant \(c>0\).
Equivalently, how many random neurons suffice to ensure that the minimum eigenvalue of the empirical feature Gram matrix exceeds a given threshold?
\end{question}
For this stronger
quantitative question, we work with the normalized data \(U\) and write \(Z=\sigma(UR)\) for the stable-rank statements below. The following holds:
\begin{theorem}[Thm. \ref{thm:homogeneous-stable-rank} informal]\label{thm:exact_rank_informal}
Let 
	\[
		K_\sigma
		=
		\E\left[
			\sigma(Ug)\sigma(Ug)^\top
		\right],
		\qquad
		\kappa_\sigma=\lambda_{\min}(K_\sigma).
	\]
denote the population kernel and its smallest eigenvalue. If
	\[
		n \ge C_{\sigma}\cdot \frac{m}{\theta^{2q+1}}\cdot \log^{2q+\frac{1}{2}}\left(c_{\sigma}\frac{m}{\theta}\right) \log\left(\frac{m}{\delta}\right),
	\]
then 
	\[
		\Pb\left(
			\lambda_{\min}(ZZ^\top)
			\ge
			\frac{n\kappa_\sigma}{4}
		\right)
		\ge
		1-\delta,
	\]
with $q$ the degree of homoegenity and $c_{\sigma}$ and $C_{\sigma}$ two constants that only depend on the activation. 
\end{theorem}
In particular, for the special case of the sign considered in \cite{DragoBCMSSB26}, \Cref{thm:exact_rank_informal} implies that there are universal constants \(c,C>0\) such that
\[
    n
    \ge
    C\frac{m\sqrt{\log(m)}}{\theta}
    \log\left(\frac{m}{\delta}\right)
    \quad\Longrightarrow\quad
    \Pb\left(
        \lambda_{\min}(ZZ^\top)
        \ge
        cn\frac{\theta}{\sqrt{\log(m)}}
    \right)
    \ge
    1-\delta.
\]
Our second result gives finite-width stable rank guarantees for worst-case projectively separated data.
Its population component---lower bounding the smallest eigenvalue
of a neural Gram matrix---has also played a central role in analyses of
optimization and memorization
\cite{DuZhaiPoczosSingh19,AllenZhuLiSong19,OymakSoltanolkotabi20,
PanigrahiShettyGoyal20,NguyenMondelliMontufar21,
KarhadkarMurrayMontufar24,LiuMX25,Song26}. In particular,
Panigrahi et al.~\cite{PanigrahiShettyGoyal20}
emphasize the role of activation
nonsmoothness, while Song~\cite{Song26} proves a sharp
\(\theta/\sqrt{\log m}\)-type bound for the continuous ReLU-derivative Gram
matrix.
Liu et al.~\cite{LiuMX25} also bound kernel spectra for higher-order activations, specifically ReLU powers.
Our proof strategy for \Cref{sec:homogeneous-stable-rank}
combines the Taylor-tail/Gershgorin approach of
\cite{PanigrahiShettyGoyal20,Song26} with Gaussian kernel expansions
\cite{DanielyFrostigSinger16,HolzmullerScholpple26} and standard
truncation and matrix-concentration arguments \cite{Tropp15}.
While we do not claim novelty for the technical ingredients of this analysis,
we bring them together in a unified, self-contained,
nonasymptotic treatment of positively homogeneous nonpolynomial
activations of arbitrary nonnegative integer degree $q$.
This synthesis makes the dependence on the activation explicit and
allows us to contrast the quantitative requirements of stable rank
lifting with our activation-uniform guarantees for exact rank lifting.
Specifically, if \(\sigma\) has homogeneity degree \(q\), then
\[
    \kappa_\sigma
    \ge
    c_\sigma
    \min\left\{
        1,\,
        \left(
            \frac{\log(\sec\theta)}{\log(m)}
        \right)^{q+\frac12}
    \right\}.
\]
Here \(c_\sigma>0\) depends only on \(\sigma\), and the minimum is interpreted as \(1\) when \(\theta=\pi/2\).

We remark that the dual hyperplane-arrangement viewpoint underlies both results: local
chamber crossings yield exact rank lifting, while the associated correlation
kernel controls stable rank lifting.

\paragraph{Exact versus stable rank lifting.}
Stable rank is a substantially stronger property than exact rank. The latter already
means that, once the random hidden layer is sampled, the same representation
can realize every scalar or vector-valued target on the dataset through a
linear readout, while stable rank lifting makes this guarantee quantitative. Indeed,
if \(Z\) has full row rank and \(\bw^\star\) is the minimum-norm solution to
\(Z\bw=y\), then
\[
    \|\bw^\star\|_2
    \le
    \frac{\|y\|_2}
         {\sqrt{\lambda_{\min}(ZZ^\top)}}.
\]
In particular, for every binary labeling \(y\in\{\pm1\}^m\),
\[
    \min_{i\in[m]}
    y_i
    \left\langle
        Z_{i,:},
        \frac{\bw^\star}{\|\bw^\star\|_2}
    \right\rangle
    =
    \frac{1}{\|\bw^\star\|_2}
    \ge
    \sqrt{
        \frac{\lambda_{\min}(ZZ^\top)}{m}
    }.
\]
Thus the same random representation linearly realizes every binary labeling
with an explicit feature-space margin guarantee, while also controlling the readout norm and sensitivity to perturbations of the features
(see \Cref{lem:min-norm-interpolant}).

The stable result quantifies the nonsingularity of the same
random-hyperplane representation.
Unlike our exact-rank bound, the stable guarantee therefore depends quantitatively on
the activation and its degree of homogeneity. This activation dependence is unavoidable: already for a
single unit input and $\sigma(t)=t_+^q$, exact rank holds with probability
$1-2^{-n}$, independently of $q$, whereas achieving
$\lambda_{\min}(ZZ^\top)\ge n\kappa_\sigma/4$ with any fixed positive
success probability requires $n=\exp(\Omega(q))$
(Appendix \ref{app:degree-dependent-separation}). The distinction is that exact
rank needs only a positive preactivation, while stable rank must capture
the contribution of rare, large preactivations to the population second
moment.

\paragraph{Proof overview: exact rank lifting.}
We next describe the main ingredients needed for the proof of the exact-rank bound. Positive homogeneity ensures that row normalization preserves rank.
First, we imagine the columns $\sigma(Ug)$ of $Z = \sigma(UR)$ to arrive sequentially, and assume that the feature columns sampled so far span a proper subspace \(V\subsetneq\mathbb R^m\). We then prove that the next column escapes \(V\) with probability \(\Omega(\theta)\). To this end, let us choose a nonzero \(a\in V^\perp\) and consider
\[
    \sum_{i=1}^m
    a_i\sigma(\langle u_i,g\rangle).
\]
We observe that the active hyperplanes \(u_i^\perp\), corresponding to \(a_i\neq0\), partition
\(\mathbb R^d\) into chambers.
For the sign activation, the displayed expression is constant within each chamber, as already noted in \cite{DragoBCMSSB26}, since crossing exactly one active hyperplane changes it by \(\pm2a_i\), so it cannot vanish at both endpoints of such a crossing. Indeed, suppose that \(g_1\) and
\(g_2\) lie in adjacent chambers. Relabeling the active indices as
\(1,\ldots,k\), we may assume that
\(\sigma(\langle u_1,g_1\rangle)=1\),
\(\sigma(\langle u_1,g_2\rangle)=-1\), and
\(\sigma(\langle u_i,g_1\rangle)=\sigma(\langle u_i,g_2\rangle)\) for
\(i=2,\ldots,k\). The two sums therefore satisfy
\[
    \sum_{i=1}^k a_i\sigma(\langle u_i,g_1\rangle)
    -
    \sum_{i=1}^k a_i\sigma(\langle u_i,g_2\rangle)
    =
    2a_1
    \neq0,
\]
since the crossed hyperplane is active, and hence they cannot both vanish. We generalize the above argument to positively homogeneous activation functions of higher powers.

For ReLU, the expression is linear on each chamber, and closely related
one-facet comparisons appear in
\cite[Lemma 2 and Appendix B]{PetzkaTrimmelSminchisescu20},
\cite[Theorem 2.1 and its proof in Section~4]{HeLiXuZheng20}, and, most
directly, in the proof of
\cite[Lemma 16 in Appendix A.3 of the full version]{FroeseGrilloSkutella25}.
These arguments exploit the change of the linear formula or of its gradient
across one ReLU hyperplane. For higher homogeneous powers, however, the expression is generally analytic
rather than linear inside a chamber. We extend the preceding observation by
showing that it cannot vanish in both chambers adjacent to an
active facet: near a point in the relative interior of that facet, only the crossed neuron changes branch, while all other terms remain analytic. If the sum vanished in both chambers,
restricting to a line transverse to the facet would
analytically continue the two branches of \(\sigma\) through the origin, forcing \(\sigma\) to be a pure monomial. 

To turn this local observation into a probability bound, we perturb a Gaussian vector while preserving its distribution. Let
\(g,h\sim\mathcal N(0,I_d)\) be independent, and set
\(\widetilde g=(g+(\theta/k)h)/\sqrt{1+(\theta/k)^2}\), where \(k\) is the number of active
hyperplanes. In words, we apply a small Gaussian displacement and then a positive rescaling,
where the latter does not change the chamber containing the endpoint. The perturbation must be large enough to cross a facet, but small enough
to avoid crossing several at once. Each active hyperplane separates the endpoints with probability \(\varphi/\pi\), while projective separation
controls the probability that two hyperplanes separate them simultaneously. The scale \(\theta/k\) balances these effects and gives
probability at least \(\theta/2\pi\) of crossing exactly one active hyperplane. On this event, the endpoints lie in adjacent chambers, so at least one breaks the proposed linear relation almost surely. Crucially, both endpoints have the same standard Gaussian distribution.
The probability that a single Gaussian weight breaks the relation is therefore at least \(\theta/4\pi\). Finally, since we have imagined revealing the columns sequentially, applying a scalar Chernoff bound gives the stated width guarantee.

\paragraph{Proof overview: stable rank lifting.}
The proof ideas are most transparent for the sign activation function, but, as we explain further below, the argument extends to all positively homogeneous nonpolynomial activation functions. When $\sigma=\sgn$, the population kernel is given explicitly by
\[
    (K_{\sgn})_{ij}
    =
    \frac{2}{\pi}
    \arcsin(\langle u_i,u_j\rangle),
\]
and its Taylor expansion $K_{\sgn} = \sum_{k=0}^\infty a_k(UU^\top)^{\circ k}$ has nonnegative coefficients $a_k$ whose tail beyond degree
\(N\) is of order \(N^{-1/2}\). We therefore retain the high-degree tail
$T_N = \sum_{k=N}^\infty a_k(UU^\top)^{\circ k}$. All omitted terms are positive semidefinite, so \(K_{\sgn}\succeq T_N\).
Its diagonal entries contain the full tail mass, whereas projective
separation suppresses every off-diagonal entry by a factor
\((\cos\theta)^N\). Choosing \(N\) so that
\((m-1)(\cos\theta)^N\le1/2\), Gershgorin's circle theorem gives
\(\kappa_{\sgn}\gtrsim\theta/\sqrt{\log(m)}\). Since sign columns have
deterministic squared norm \(m\), matrix Chernoff transfers this population
gap directly to the empirical Gram matrix.

For a general positively homogeneous nonpolynomial activation of degree $q$, we compute the ordinary Taylor
expansion of its scalar Gaussian correlation kernel. These Taylor coefficients are the squares of the orthonormal Gaussian Hermite coefficients of \(\sigma\) \cite{DanielyFrostigSinger16,HolzmullerScholpple26}.
Its coefficients remain
nonnegative, but their tail is now of order \(N^{-q-\frac12}\). The same
diagonally dominant tail argument gives the stated population bound. Finally,
for unbounded activations such as ReLU, we truncate the rare columns
containing a large preactivation, control the resulting population bias, and
then apply matrix Chernoff to the bounded truncated columns.

\subsection{Related Work}\label{sec:related-work}

\paragraph{Random hidden layers and random features.}
Models with randomly sampled and frozen hidden parameters include random-basis
networks \cite{IgelnikPao95}, extreme learning machines
\cite{HuangZhuSiew06}, and echo-state and liquid-state networks
\cite{Jaeger01,MaassNatschlaegerMarkram02,LukoseviciusJaeger09}. Random
features were introduced as finite-dimensional kernel approximations
\cite{RahimiRecht07,RahimiRecht08}, followed by work on their approximation,
statistical, and computational properties
\cite{Bach17,RudiRosasco17,AvronEtAl17,LiuEtAl22}, and on nonlinear component
analysis \cite{LopezPazEtAl14}. Earlier full-rank results for extreme learning
machines are mainly qualitative and typically rely on genericity, random
biases, or smoothness assumptions \cite{HuangZhuSiew06}. We instead seek
high-probability width bounds for a finite, bias-free Gaussian feature matrix
on fixed projectively separated data.

\paragraph{Spectral Distribution and Asymptotic Random Feature Analysis.}
A prominent line of work analyzes the spectral properties and minimum eigenvalues of random feature and kernel Gram matrices using Random Matrix Theory (e.g., \cite{pmlr-v258-dandi25a,misiakiewicz2022spectruminnerproductkernelmatrices,MeiMontanari22}). These results typically establish deterministic equivalents and empirical spectral distributions in proportional asymptotic regimes ($n, d, m \to \infty$ at fixed ratios) under smooth or isotropic data distributions. In contrast, our work operates in a strictly non-asymptotic setting, establishing dimension-free exact and stable rank lower bounds that hold for fixed sample sizes $n$ under worst-case, projectively separated data.

\paragraph{Interpolation, memorization, and rank lifting.}
Interpolation by overparameterized models is closely tied to questions of
capacity and generalization
\cite{ZhangEtAl17,BelkinEtAl19,LiangRakhlin20}. Random-feature interpolation
has been studied largely in asymptotic or average-case settings
\cite{MeiMontanari22,MontanariZhong22}, while other work bounds the
memorization capacity of threshold and ReLU networks
\cite{BaldiVershynin19,Vershynin20,Daniely20}. In nonasymptotic regimes, \cite{bubeck2020networksizeweightssize} study the network size and
weight magnitudes needed for memorization under general-position
and, for some results, additional well-dispersedness assumptions. On the other hand, in \cite{dirksen2022separation,ghosal2022randomly}, the authors investigate the separation properties of randomly initialized, single or two-layer ReLU networks with uniform bias in feature space. Differently from this line of work, rank lifting asks a more specific question: when does one frozen random layer already make arbitrary
labels linearly representable? The closest work is \cite{DragoBCMSSB26},
which studies the special case of a Gaussian sign layer through its spherical hyperplane arrangement. Specifically, \cite{DragoBCMSSB26} uses a global chamber-volume argument to prove a width requirement that scales exponentially with the number of points and as \(\sqrt d\) with the ambient dimension,
with sharper results in dimensions two and three. Our local subspace-escape argument avoids these volume estimates and gives a dimension-free bound linear in \(m\), for fixed \(\theta\).
A brief, yet more detailed comparison with \cite{DragoBCMSSB26} is given in \Cref{rem:dragoetal_comparison}.

\paragraph{Hyperplane arrangements and nonpolynomial activations.}
Hyperplane sign patterns and region counts have classical roots
\cite{Cover65,Zaslavsky75} and are widely used to study the expressivity of
piecewise-linear networks
\cite{MontufarEtAl14,RaghuEtAl17,SerraEtAl18}. Rather than counting chambers or estimating their volumes \cite{DragoBCMSSB26},
our exact-rank proof studies one crossing of an active facet and shows that the associated linear relation cannot hold identically in both adjacent chambers.
Nonpolynomiality is known to
characterize universal approximation by ridge-function networks
\cite{LeshnoEtAl93,Pinkus99}. In this work, it has a local role: it prevents the two
homogeneous branches of the activation from joining analytically across the
facet.

\paragraph{Neural kernels and minimum-eigenvalue bounds.}
Infinite-width networks induce Gaussian-process and dot-product kernels
\cite{Neal96,Williams97,ChoSaul09}, while neural tangent kernels describe
linearized training \cite{JacotGabrielHongler18}. Their dependence on input
correlations is captured by the dual-activation formalism
\cite{DanielyFrostigSinger16}. Lower bounds on neural Gram-matrix eigenvalues
are central to convergence analyses
\cite{DuZhaiPoczosSingh19,AllenZhuLiSong19,OymakSoltanolkotabi20}, and
activation-sensitive, deep-ReLU, arbitrary-spherical-data, and qualitative
positivity results appear in
\cite{PanigrahiShettyGoyal20,NguyenMondelliMontufar21,
KarhadkarMurrayMontufar24,CarvalhoCostaMouraoOliveira25}.
Most closely related to our population argument, Liu et al.~\cite{LiuMX25} bound kernel spectra for higher-order activations and specifically ReLU-powers, and Song \cite{Song26} proves a sharp
worst-case bound for the continuous ReLU-derivative Gram matrix by isolating a
diagonally dominant high-degree tail of a positive Hadamard-power expansion.
That matrix is the hidden-weight component of the ReLU NTK, rather than the
plain ReLU feature covariance present in this work.

\paragraph{Power-series and random-matrix methods.}
Nonnegative power expansions of spherical positive definite kernels originate
in \cite{Schoenberg42} and underlie modern spectral and RKHS analyses of
neural dot-product kernels
\cite{DanielyFrostigSinger16,BiettiBach21,ChenXu21,ScetbonHarchaoui21,
MurrayJinBowmanMontufar23,HolzmullerScholpple26}. We use this structure for a
worst-case finite Gram matrix: discarding low degrees leaves a tail whose
diagonal mass dominates its off-diagonal entries. A complementary literature
studies nonlinear random-feature spectra in proportional asymptotic regimes
\cite{PenningtonWorah17,LouartLiaoCouillet18,FanWang20}. Our guarantees are
instead nonasymptotic, dimension-free, and valid for deterministic data under
projective separation.

%% file: sections/preliminaries.tex
\section{Setup and Assumptions}\label{sec:preliminaries}

In this section, we provide the formal setup and required assumptions for the rest of this paper, recalling some of the notions introduced earlier.

\paragraph{Projective separation.} Let \(X\in\mathbb R^{m\times d}\) have nonzero rows
\(x_1^\top,\ldots,x_m^\top\), and let
\(R\in\mathbb R^{d\times n}\) have independent standard Gaussian columns
\(g_1,\ldots,g_n\sim \mathcal N(0,I_d)\). The random feature matrix is defined as $Z=\sigma(XR)\in\mathbb R^{m\times n}$, where \(\sigma: \R \to \R\) is an activation function applied entrywise. The \(i\)-th row of \(Z\) is the random
representation of \(x_i\), while the \(j\)-th column records the response of the \(j\)-th random neuron on all \(m\) data points. Throughout this work, we write $u_i=x_i/\|x_i\|_2$ and denote by $\theta_{ij}=\arccos(\langle u_i,u_j \rangle)$ the angle between vectors $x_i,x_j$ (equivalently $u_i, u_j$) for $i \neq j$ and make the following assumption:
\begin{assumption}[Projective separation]
\label{asmpt:separation}
For some \(0<\theta\le\pi/2\), the rows of \(X\) satisfy
\[
    \theta
    \le
    \theta_{ij}
    \le
    \pi-\theta
    \qquad
    \text{for all }i\neq j.
\]
\end{assumption}
The quantity \(\theta\) is a lower bound on the minimum distance between two rows viewed as
directions in projective space: it rules out pairs that are nearly parallel or
nearly antiparallel. The two-sided separation is unavoidable for a theorem
that includes the sign activation, since parallel inputs produce identical
sign features and antiparallel inputs produce opposite sign features.
More generally, \(\theta\) quantifies how difficult it is for a random
hyperplane to distinguish two projective directions. Furthermore, note that, denoting by $\rho_{ij} = |\langle u_i,u_j\rangle|$ and by $\rho = \cos \theta$, the above assumption can be written as $\rho_{ij} \le \rho$ for all $i \neq j$.

\paragraph{Activation functions.}
Our main results concern the class of positively homogeneous activations $\sigma$, i.e., such that, for some integer \(q\ge0\),
\[
    \sigma(ct)=c^q\sigma(t)
    \qquad
    \text{for every }c>0\text{ and }t\neq0.
\]
Here we allow homogeneity of arbitrary nonnegative integer degree: for example, sign is \(0\)-homogeneous,
ReLU and leaky ReLU are \(1\)-homogeneous, and the ReLU powers
\(t\mapsto(t)_+^q\) are \(q\)-homogeneous. For $q=0$, we use the convention $t_+^0\coloneqq\mathbbm 1\{t>0\}$.
Note that the value of \(\sigma(0)\) is irrelevant, since Gaussian preactivations are nonzero almost surely.

Positive homogeneity gives additional properties that are particularly useful for both the exact as well as the stable rank questions. Indeed, letting \(U\in\mathbb R^{m\times d}\) be the matrix with rows \(u_i^\top\) and, thus, writing $X=DU$ for $D=\diag(\|x_1\|_2,\ldots,\|x_m\|_2)$, we have
\[
    \sigma(XR)=D^q\sigma(UR)\qquad\text{almost surely}.
\]
Since \(D^q\) is invertible, the matrices \(\sigma(XR)\) and \(\sigma(UR)\) have the same rank. Thus the exact-rank problem depends only on the normalized directions \(u_1,\ldots,u_m\), not on the lengths of the data vectors.
For stable rank lifting, the same identity transfers any lower bound on the smallest eigenvalue of $\sigma(UR)\sigma(UR)^\top$ to $\sigma(XR)\sigma(XR)^\top$ with an additional factor $\min_{i\in[m]}\|x_i\|_2^{2q}$.

There is, however, an immediate obstruction for a subclass of positively homogeneous activations, namely, pure polynomials (up to the value at zero). Indeed, a positively homogeneous activation can be polynomial only if it is a monomial
\(\sigma(t)=ct^q\). Such activations though cannot lift rank in a
dimension-independent manner: there exist projectively separated datasets $X$ such that, if \(m>\binom{d+q-1}{q}\) and $\sigma$ is a monomial, then no number of random features can make \(\sigma(XR)\) full row rank, and thus $\Pb(\rank(\sigma(XR))=m)=0$ (see \Cref{obs:monomial-obstruction}).

Once we exclude the polynomial subclass, which we show to be the only obstruction, we observe that a positively homogeneous activation can be written as $\sigma(t) = c_+t^q_+ + c_-(-t)^q_+$ for $t\neq0$, where \(t_+ \coloneqq \max\{0,t\}\), \(c_+\coloneqq\sigma(1), c_-\coloneqq\sigma(-1)\) and with $c_-\neq(-1)^q c_+$. Indeed, a positively homogeneous activation agrees on \(\mathbb R\setminus\{0\}\) with a polynomial if and only if $q \in\mathbb Z_{\ge0}$ and $c_-=(-1)^q c_+$, in which case \(\sigma(t)=c_+t^q\) away from the origin.

To summarize, we make the following assumption on activation functions, and, for succinctness, we denote by $M \coloneqq \max\{|c_+|,|c_-|\}$ the maximum of the two coefficients in absolute value:
\begin{assumption}[Positively homogeneous non-polynomial activations]
\label{asmpt:activations}
For some integer \(q \ge 0\) and all $t \in \R\setminus\{0\}$, the activation $\sigma: \R \to \R$ satisfies
\[
    \sigma(t) = c_+t^q_+ + c_-(-t)^q_+,
\]
where \(c_+=\sigma(1), c_-=\sigma(-1)\) and with $c_-\neq(-1)^q c_+$.
The value of $\sigma(0)$ is arbitrary.
\end{assumption}

Note that the sign function $\sgn(t) = \mathbbm 1\{t > 0\} - \mathbbm 1\{t < 0\}$ is an instantiation of the above with $q=0, c_+=1,c_-=-1$, the Heaviside step function $H(t)=\mathbbm 1\{t > 0\}$ also is with $q=0, c_+=1,c_-=0$, and finally the ReLU function $\text{ReLU}(t) = t_+$ with $q=1, c_+=1,c_-=0$.

\paragraph{Hyperplane arrangements and polyhedral fans.}
Given nonzero vectors $\{u_1,\ldots,u_k\} \subseteq \R^d$, a collection of hyperplanes $\{\{x: \langle u_i, x \rangle = b_i\}: i \in [k]\}$ for values $b_i \in \R$ is commonly referred to as a hyperplane arrangement. If all $b_i = 0$, i.e., all hyperplanes contain the origin, the hyperplane arrangement $\mathcal{F}= \{u_i^\perp: i \in [k]\}$ is called a \emph{central hyperplane arrangement}. Its \emph{chambers} are the connected components of $\mathbb R^d\setminus\bigcup_{i=1}^k u_i^\perp$. These are open, full-dimensional polyhedral cones; their closures, together with all their faces, form a \emph{polyhedral fan}. We call two chambers \emph{adjacent} if the intersection of their closures $C_1, C_2$ is a $(d-1)$-dimensional face of both, i.e., if there exists a hyperplane $u_i^\perp$ in the arrangement such that $C_2\cap u_i^\perp = C_1 \cap u_i^\perp$ is $(d-1)$-dimensional. See also \Cref{fig:poly_fan} for an illustration.

\begin{figure}
    \centering
    \includegraphics[width=0.3\linewidth]{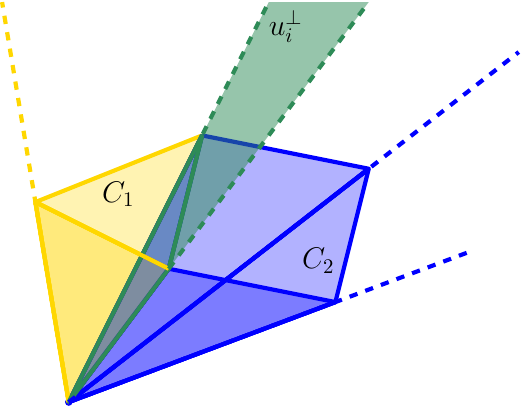}
    \caption{Two adjacent chambers of a three-dimensional polyhedral fan.}
    \label{fig:poly_fan}
\end{figure}

%% file: sections/exact_rank.tex
\section{Exact Rank Lifting}
\label{sec:exact-rank}

The purpose of this section is to study exact rank without requiring a
quantitative lower bound on the smallest singular value. This distinction
allows us to avoid both population-kernel estimates and matrix concentration.
Instead, we show directly that every proper subspace of \(\mathbb R^m\) is
escaped by a single random feature column with probability
\(\Omega(\theta)\). Revealing the columns sequentially then gives full rank
after
\[
    n=O\left(\frac{1}{\theta}\max\left\{m,\log\left(\frac{1}{\delta}\right)\right\}\right)
\]
samples. These statements concern exact rank only and do not assert a lower bound on \(\lambda_{\min}(\sigma(XR)\sigma(XR)^\top)\). Indeed, \Cref{prop:degree-dependent-separation} shows that the width $n$ required for stable rank lifting can grow exponentially with the homogeneity degree $q$, even when the exact-rank width is independent of that degree.

\begin{theorem}[Exact rank lifting]
\label{thm:exact-rank-homogeneous}
Let \(X\in\mathbb R^{m\times d}\) satisfy
\Cref{asmpt:separation}, let
\(R\in\mathbb R^{d\times n}\) have independent \(\mathcal N(0,I_d)\) columns, and let
\(\sigma\) satisfy \Cref{asmpt:activations}. Then, for every \(\delta\in(0,1)\), if
\[
    n\ge
    \frac{8\pi}{\theta}\max\left\{
        m,
        4\log\frac1\delta
    \right\}
\]
then
\[
    \Pb\left(\rank(\sigma(XR))=m\right)\ge 1-\delta.
\]
\end{theorem}

Recall that since this part is concerned with the exact rank problem and we are considering homogeneous activations, the norms of the $x_i$ are irrelevant, and we thus only use their normalized versions $u_i$.
\par
The proof of the theorem will be split up into two parts. First, in \Cref{lem:homogeneous-anti-concentration} we will prove that the probability of $\sigma(Ug)$ escaping any fixed proper subspace $V\subsetneq\R^m$ can be lower-bounded by $\theta/4\pi$. We then use this estimate to lower-bound the probability that $\rank(\sigma(UR)) = m$ by iteratively revealing the columns $\sigma(Ug)$ and lower-bounding their conditional probability of escaping the span of the previously revealed columns.

\begin{lemma}[Uniform linear anti-concentration]
\label{lem:homogeneous-anti-concentration}
Let \(g\sim \mathcal N(0,I_d)\) be a standard normal vector and let $V\subsetneq\R^m$ be a subspace of $\R^m$. Consider \(U=(u_1,\ldots,u_m)^\top\) with unit rows satisfying
\Cref{asmpt:separation} and
\(\sigma\) satisfying \Cref{asmpt:activations}. Then,
\begin{align}\label{eq:single_subspace_escape}
    \Pb\left(\sigma(Ug)\notin V\right)\ge\frac{\theta}{4\pi}.
\end{align}
\end{lemma}

\begin{proof}
Note first that since every proper subspace of $\R^m$ is contained in a hyperplane, it suffices to consider $V=\{x\in\R^m:\langle a,x\rangle=0\}$ for some $a\neq0$. Thus, the probability of interest becomes
\begin{align*}
    \Pb\left(\sum_{i=1}^m a_i\sigma(\langle u_i,g\rangle)\neq0\right).
\end{align*}
Moreover, assume without loss of generality that non-zero entries of $a\in\R^m$ are at indices $\{1,\ldots,k\}$ for some $1\le k\le m$ and define $F(g)=\sum_{i\in[k]}a_i\sigma(\langle u_i,g\rangle)$. If only one index is non-zero, the assumptions on $\sigma$ and $U$ directly imply a lower bound of $1/2$ on the probability, since at least one of $c_+,c_-$ is nonzero; hence, we assume $k\ge2$.
Let $\mathcal F=\{u_i^\perp:i\in[k]\}$ be the central arrangement of active hyperplanes. Observe that within every chamber of the arrangement the signs of $\langle u_i,g\rangle$, $i\in[k]$, do not change. Therefore, the function $F$ restricted to any fixed chamber $C$ is a linear combination of positive linear forms to the power \(q\) and in particular polynomial since $q$ is an integer. Since a nonzero polynomial has a zero set of Lebesgue measure zero, on each chamber either $F$ vanishes identically or its zero set has measure zero.

Now, notice that our assumptions on $F$ prohibit it from vanishing on any two adjacent chambers. Indeed, assume for a contradiction that it does and let $u_i^\perp$ denote the hyperplane containing their shared facet. Fix a point $g^*\in u_i^\perp$ in the relative interior of this facet and outside all other active hyperplanes, i.e., $\langle u_j,g^*\rangle\neq0$ for all $j\in[k]\setminus\{i\}$, and consider the function $F(g^*+tu_i)$.
For sufficiently small positive and negative $t$, the points $g^*+tu_i$ lie in the two respective chambers. No other active hyperplane is crossed, so the signs of $\langle u_j,g^*+tu_i\rangle$, $j\in[k]\setminus\{i\}$, stay the same. (See also \Cref{fig:smooth_line_segment_in_poly_fan} for an illustration.)
This implies that the sum $J(t)=\sum_{j\in[k]\setminus\{i\}}a_j\sigma(\langle g^*+tu_i,u_j\rangle)$ is again real analytic on a small segment $t\in(-\varepsilon,\varepsilon)$. Since $F(g^*+tu_i)=a_i\sigma(\langle g^*+tu_i,u_i\rangle)+J(t)=0$ for $0<|t|<\varepsilon$, we deduce that
\begin{align*}
    \sigma(t)=\sigma(\langle g^*+tu_i,u_i\rangle)=-\frac{J(t)}{a_i},
    \qquad {0<|t|<\varepsilon}.
\end{align*}
Since $J(t)$ is real analytic on $(-\varepsilon,\varepsilon)$, the two branches of $\sigma$ would admit a common real-analytic extension across zero. Their one-sided $q$-th derivatives, $q!c_+$ and $(-1)^q q!c_-$, would therefore agree, forcing $c_-=(-1)^q c_+$, contrary to \Cref{asmpt:activations}. For $q=0$, the same argument compares the one-sided values.

Knowing that $F$ cannot vanish on two adjacent chambers helps us translate the probability of escaping any fixed hyperplane into the probability that two correlated Gaussians fall into adjacent chambers. Indeed, let $h\sim\mathcal N(0,I_d)$ be another standard normal vector, independent of $g$, and set
\begin{align*}
        \varphi=\arctan(\theta/k),\qquad
    \Tilde{g}=\cos(\varphi)g+\sin(\varphi)h =\frac{g+(\theta/k)h}{\sqrt{1+(\theta/k)^2}}.
\end{align*}

Conditionally on \(g\), the vector \(\widetilde g\) is Gaussian with mean \(\cos(\varphi)\,g\) and covariance \(\sin^2(\varphi)\,I_d\). Thus the coupling explores directions around \(g\).
Let us denote by
\begin{align*}
        N=\left|\left\{i\in[k]: \sgn(\langle u_i,g\rangle)\neq \sgn(\langle u_i,\Tilde g\rangle)\right\}\right|
\end{align*}

the number of sign changes. If $N=1$, the segment joining $g$ and $\Tilde g$ crosses exactly one active hyperplane: every other active linear form keeps its strict sign along the segment. Hence the endpoints lie in adjacent chambers. There are finitely many chambers, and the zero sets in chambers where $F$ is not identically zero have Gaussian measure zero. Both endpoints avoid these sets and the active hyperplanes almost surely. Thus, on $N=1$, at least one of $F(g)$ and $F(\Tilde g)$ is nonzero almost surely. Since $g$ and $\Tilde g$ are both standard normal vectors, we can deduce
\begin{align*}
    2\Pb(F(g)\neq0)
    &=\Pb(F(g)\neq0)+\Pb(F(\Tilde g)\neq0)\\
    &\ge\Pb\bigl(\{F(g)\neq0\}\cup\{F(\Tilde g)\neq0\}\bigr)
    \ge\Pb(N=1).
\end{align*}
The remaining technical estimate, $\Pb(N=1)\ge\theta/2\pi$, is proved in \Cref{lem:single-crossing} in Appendix~\ref{app:omitted-rank}. The result then follows directly by
\begin{align*}
        \Pb\left(\sum_{i=1}^m a_i\sigma(\langle u_i,g\rangle)\neq0\right) \ge\frac12\Pb(N=1)\ge\frac{\theta}{4\pi}. \qquad \qedhere
\end{align*}

\end{proof}

\begin{figure}
    \centering
    \includegraphics[width=0.4\linewidth]{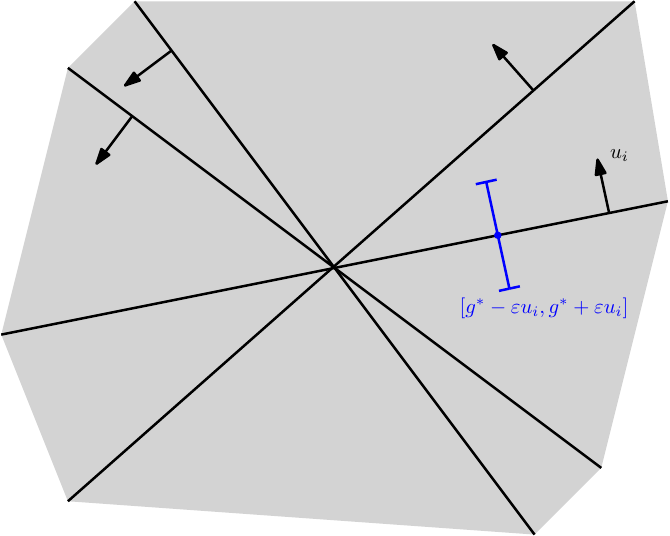}
    \caption{The hyperplane arrangement $\mathcal F$ generated by the vectors $u_1,\ldots,u_k$, the point $g^*$ in the relative interior of a shared facet, and the line segment $[g^*-\varepsilon u_i,g^*+\varepsilon u_i]$. 
    }
    \label{fig:smooth_line_segment_in_poly_fan}
\end{figure}

It is worth remarking that the proof above implicitly shows that the distribution of $\sigma(Ug)$ is not supported on any proper subspace $V\subsetneq\R^m$. Without excluding monomials in the assumptions on $\sigma$, this need not hold. Indeed, for $\sigma(t)=t$ and $\rank(U)<m$, he image of $g\mapsto Ug$ is contained in every hyperplane $\{x\in\R^m:\langle x,a\rangle=0\}$ with nonzero $a\in\ker(U^\top)$, meaning that a bound of the form \eqref{eq:single_subspace_escape} cannot exist in this setting.

We conclude with the proof of \Cref{thm:exact-rank-homogeneous} using an elementary observation that converts uniform subspace escape into
a full rank guarantee (see, e.g.,
\cite[Definition 26 and Lemma 27]{karhadkar2023mildly}).
\begin{proof}[Proof of \Cref{thm:exact-rank-homogeneous}]
By \Cref{lem:homogeneous-anti-concentration}, every proper subspace is escaped by
\(\xi=\sigma(Ug)\) with probability at least $p=\theta/4\pi$. We claim that $n$ independent copies of \(\xi\), denoted as \(\xi_1,\ldots,\xi_n\), satisfy
\begin{align}\label{eq:escape}
    \Pb\left(
        \Span\{\xi_1,\ldots,\xi_n\}\neq\mathbb R^m
    \right)
    \le
    \exp\left(-\frac{np}{8}\right),
\end{align}
provided \(np\ge2m\), which holds here since \(n\ge8\pi m/\theta\). Assuming \eqref{eq:escape} for a moment, the event $\{\rank(\sigma(UR))<m\}$ is equivalent to $\Span\{\xi_1,\ldots,\xi_n\}\neq\mathbb R^m$, yielding
\[
    \Pb\left(\rank(\sigma(UR))<m\right)
    \le
    \exp\left(-\frac{n\theta}{32\pi}\right).
\]
Requiring this bound to be at most \(\delta\) and using rank invariance under row normalization gives the main statement.

To prove \eqref{eq:escape}, let the vectors be revealed sequentially and define $V_j=\Span\{\xi_1,\ldots,\xi_j\}$, with $V_0=\{0\}$. We introduce the indicator $Y_j$ that the $j$-th copy is not spanned by previous vectors until full rank is reached: namely, \(Y_j=\mathbf 1_{\{\xi_j\notin V_{j-1}\}}\) if \(V_{j-1}\neq\mathbb R^m\), and \(Y_j=1\) once \(V_{j-1}=\mathbb R^m\). Conditional on the previously revealed vectors, and since $\Pb(\xi\notin V)\ge p$,
\[
    \Pb(Y_j=1\mid\xi_1,\ldots,\xi_{j-1})\ge p.
\]
Thus, $\E\left[2^{-Y_j}\,\middle|\,\xi_1,\ldots,\xi_{j-1}\right]\le1-p/2$, and iterating gives $\E\left[2^{-\sum_{j=1}^nY_j}\right]\le(1-p/2)^n$. Finally, since the event $\{V_n\neq\mathbb R^m\}$ is equivalent to $\{\sum_{j=1}^nY_j<m\}$, we have
\[
    \Pb\left(V_n\neq\mathbb R^m\right)
    =\Pb\left(\sum_{j=1}^nY_j<m\right)
    \le
    2^m\left(1-\frac p2\right)^n
    \le
    \exp\left(m\log2-\frac{np}{2}\right).
\]
Since \(np\ge2m\), we have \(m\log2\le3np/8\), which concludes the proof of \eqref{eq:escape}, and yields the theorem.
\end{proof}

\begin{remark}\label{rem:dragoetal_comparison}
It is interesting to compare this result with \cite{DragoBCMSSB26} a bit more closely. Their argument combines global chamber-volume estimates with the non-vanishing property across adjacent chambers to lower-bound the subspace escape probability for sign features. The resulting width requirement scales exponentially with the number of points and as $\sqrt d$ with the ambient dimension, with sharper results in dimensions two and three. In contrast, our approach sidesteps global chamber volume estimates entirely, using an isolated-facet crossing between Gaussian endpoints with the same marginal distribution. This yields a dimension-free bound linear in the number of points for fixed $\theta$. Within the positively homogeneous class considered here, the first part of \Cref{lem:homogeneous-anti-concentration} extends the local non-vanishing argument to every nonpolynomial activation, while monomials can fail to lift rank by \Cref{obs:monomial-obstruction} and are the only such obstruction.

\end{remark}

%% file: sections/stable_rank.tex
\section{Stable Rank Lifting}
\label{sec:homogeneous-stable-rank}

In this section, we establish stable rank lifting for every positively homogeneous nonpolynomial activation satisfying \Cref{asmpt:activations}. By positive homogeneity, bounds for \(X\) follow by diagonal rescaling from bounds for
the row-normalized matrix \(U\), as explained in \Cref{sec:preliminaries}. We therefore write
\[
    Z=\sigma(UR)
    \qquad\text{and}\qquad
    K=\E_{g\sim \mathcal N(0,I_d)}
    \bigl[\sigma(Ug)\sigma(Ug)^\top\bigr],
\]
and let \(\kappa_\sigma=\lambda_{\min}(K)\).
The proof follows two steps: we show
that, once the minimum eigenvalue \(\kappa_\sigma\) of the population kernel \(K\) is positive, the empirical
feature Gram matrix retains a constant fraction of it. Since a general
homogeneous activation may be unbounded, we truncate the rare feature columns
containing a large preactivation before applying matrix Chernoff. We then
lower bound \(\kappa_\sigma\) by expanding the scalar Gaussian correlation
kernel in an ordinary Taylor series and retaining a diagonally dominant
high-degree tail. These techniques are similar to those in \cite{PanigrahiShettyGoyal20, Song26}, and here they are also extended to positively homogeneous functions.

\begin{theorem}[Stable rank lifting]
\label{thm:homogeneous-stable-rank}
Let \(X\in\mathbb R^{m\times d}\) satisfy
\Cref{asmpt:separation}, let
\(R\in\mathbb R^{d\times n}\) have independent \(\mathcal N(0,I_d)\) columns, and let
\(\sigma\) satisfy \Cref{asmpt:activations}. There exists a universal constant \(C>0\) such that, for every \(\delta\in(0,1)\), if
\[
    n \ge C^{q+1} \cdot \frac{M^2}{(c_+-(-1)^qc_-)^2} \cdot \frac{m}{\theta^{2q+1}} \cdot \log^{2q+\frac{1}{2}}\left(\max\left(2, \frac{M^2}{(c_+-(-1)^qc_-)^2} \right) \frac{m}{\theta}\right) \log\left(\frac{m}{\delta}\right),
\]
then
\[
    \Pb\left(
        \lambda_{\min}(ZZ^\top)
        \ge
        \frac{n\kappa_\sigma}{4}
    \right)
    \ge
    1-\delta.
\]
\end{theorem}

We begin by stating the concentration step, assuming for the moment that the
minimum eigenvalue \(\kappa_\sigma\) of the population kernel \(K\) is known.

\begin{restatable}[Concentration of the empirical kernel]{proposition}{prophomogeneousconcentration}
\label{prop:homogeneous-concentration}
Let \(\kappa_\sigma=\lambda_{\min}(K)>0\). For every
\(\delta\in(0,1)\),
\[
\begin{aligned}
    n
    \ge{}
    \frac{4^{q+1}M^2}{1-\log(2)}
    \frac{m}{\kappa_\sigma}\cdot
    \log^q\left(
        \frac{
            2^{q+\frac32}M^2m^2
        }{
            \kappa_\sigma
        }
        \left(
            \frac{\Gamma(2q+\frac12)}{\sqrt{\pi}}
        \right)^{1/2}
    \right)
    \log\left(\frac{m}{\delta}\right)
\end{aligned}
\]
implies
\[
    \Pb\left(
        \lambda_{\min}(ZZ^\top)
        \ge
        \frac{n\kappa_\sigma}{4}
    \right)
    \ge
    1-\delta.
\]
\end{restatable}

We defer the above proposition's proof to Appendix \ref{app:omitted-rank}. Write \(z_j=\sigma(Ug_j)\). The high-level idea consists of truncating the random matrices $z_j z_j^\top$ composing $ZZ^\top=\sum_{j=1}^n z_jz_j^\top$ to bound their spectral norm, choosing the truncation threshold such that the expected truncated kernel retains a minimum eigenvalue of at least $\kappa_\sigma/2$. We then apply the matrix Chernoff bound to the sum of these independent, bounded positive semidefinite matrices to establish the high-probability lower bound on $\lambda_{\min}(ZZ^\top)$.

It remains to lower bound the population quantity
\(\kappa_\sigma=\lambda_{\min}(K)\). The key
is an ordinary Taylor expansion of the scalar Gaussian correlation kernel, followed by a Taylor-tail truncation and Gershgorin argument, which we recall next.

\begin{lemma}[Gershgorin circle theorem]\label{lem:gershgorin}
    Let \(V=(v_{ij})_{i,j=1}^m\in\R^{m\times m}\). Then every eigenvalue
    \(\lambda\in\mathbb C\) of \(V\) belongs to at least one of the discs
    \[
        \left\{z\in\mathbb C: |z-v_{ii}|\le \sum_{j\neq i}|v_{ij}|\right\},
        \qquad i=1,\ldots,m.
    \]
    In particular, if \(V\) is real symmetric, then every eigenvalue is real and lies
    in one of the corresponding real intervals.
\end{lemma}

Before proceeding further, let us point out a useful rewrite of the population kernel via a Taylor series. For brevity, we set
\[
    \Delta_q
    \coloneqq
    (c_++c_-)^2\sin^2\left(\frac{\pi q}{2}\right)
    +
    (c_+-c_-)^2\cos^2\left(\frac{\pi q}{2}\right) > 0
\]
where the strict positivity of \(\Delta_q\) holds by \Cref{asmpt:activations}. Since \(q\) is an integer, \(\Delta_q=(c_+-(-1)^qc_-)^2\). The Taylor representation follows from the Gaussian Hermite expansion \cite{DanielyFrostigSinger16,HolzmullerScholpple26}; the coefficients below are the squares of the coefficients of \(\sigma\) in the orthonormal Gaussian Hermite basis. The proof of the explicit coefficient formula and tail estimate is deferred to Appendix \ref{app:omitted-rank}.

\begin{restatable}[Taylor expansion of the population kernel]{lemma}{lemhomogeneoustaylor}
\label{lem:homogeneous-taylor}
Let \((\zeta_1,\zeta_2)\) be a centered Gaussian pair with unit variances and correlation \(t\in[-1,1]\), and define
\(\psi_\sigma(t)=\E[\sigma(\zeta_1)\sigma(\zeta_2)]\). Then
\[
    \psi_\sigma(t)
    =
    \sum_{k=0}^\infty a_kt^k,
    \qquad
    a_k
    =
    \frac{
        \bigl(c_++(-1)^kc_-\bigr)^2
        2^{k-q-2}\Gamma(q+1)^2
    }{
        k!\,
        \Gamma\left(\frac{q-k+2}{2}\right)^2
    }
    \ge0.
\]
At a pole of the Gamma function in the denominator, the corresponding coefficient is interpreted as zero.
Moreover, for every integer \(N\ge1\),
\[
    \sum_{k=N}^\infty a_k
    \ge
    \frac{
        \Gamma(q+1)^2\Delta_q
    }{
        4\sqrt{2}\pi^{3/2}
        (q+\frac12)(q+5)^{q+\frac12}
    }
    \frac{1}{N^{q+\frac12}}.
\]
\end{restatable}

We finally lower bound the minimum eigenvalue of the population kernel.

\begin{proposition}[Lower bound on population kernel minimum eigenvalue]
\label{prop:homogeneous-min-eigenvalue}
Under \Cref{asmpt:separation} and \Cref{asmpt:activations},
\[
    \kappa_\sigma
    \ge
    \frac{
        \Gamma(q+1)^2\Delta_q
    }{
        2^{2q+\frac92}\pi^{3/2}
        (q+\frac12)(q+5)^{q+\frac12}
    }
    \frac{\theta^{2q+1}}
         {\log^{q+\frac12}(2m)}.
\]
\end{proposition}

\begin{proof}
Let \(G=UU^\top\), and so \(G\succeq0\), \(G_{ii}=1\), and
\(|G_{ij}|\le\rho=\cos\theta\) for every \(i\neq j\) by projective
separation. For every \(i,j\in[m]\), the pair
\((\langle u_i,g\rangle,\langle u_j,g\rangle)\) is centered Gaussian with
unit variances and correlation \(G_{ij}=\langle u_i,u_j\rangle\). Hence, by
the definition of \(\psi_\sigma\), \(K_{ij}=\psi_\sigma(G_{ij})\). Since
\Cref{lem:homogeneous-taylor} gives
\(\psi_\sigma(t)=\sum_{k=0}^\infty a_kt^k\) for \(t\in[-1,1]\), we obtain
entrywise
\[
    K
    =
    \sum_{k=0}^\infty a_kG^{\circ k},
\]
where \(G^{\circ k}\) denotes the entrywise \(k\)-th power and
\(G^{\circ0}=\mathbf 1\mathbf 1^\top\). Since \(G\succeq0\), the Schur
product theorem gives \(G^{\circ k}\succeq0\) for every \(k\ge0\);
also \(a_k\ge0\), so every term in the expansion is positive semidefinite. Let us now fix an integer \(N\ge1\), and define the Taylor tail
\[
    T_N
    =
    \sum_{k=N}^\infty a_kG^{\circ k}.
\]
Clearly, all omitted terms are positive semidefinite, so \(K\succeq T_N\), and,
because \(G_{ii}=1\), every diagonal entry of \(T_N\) equals
\(\sum_{k=N}^\infty a_k\). On the other hand, for \(i\neq j\), projective separation gives $|(T_N)_{ij}| \le \rho^N \sum_{k=N}^\infty a_k$.
By \Cref{lem:gershgorin}, the \(i\)-th Gershgorin radius is at most
\((m-1)\rho^N\sum_{k=N}^\infty a_k\), and therefore
\[
    \lambda_{\min}(K)
    \ge
    \lambda_{\min}(T_N)
    \ge
    \left(
        1-(m-1)\rho^N
    \right)
    \sum_{k=N}^\infty a_k.
\]
Assume first that \(m\ge2\) and \(0<\theta<\pi/2\), and choose $N = \left\lceil\log(2(m-1))/\log(1/\rho)\right\rceil$ so that \((m-1)\rho^N\le1/2\). By
\Cref{lem:homogeneous-taylor},
\[
    \sum_{k=N}^\infty a_k
    \ge
    \frac{
        \Gamma(q+1)^2\Delta_q
    }{
        4\sqrt{2}\pi^{3/2}
        (q+\frac12)(q+5)^{q+\frac12}
    }
    \frac{1}{N^{q+\frac12}}.
\]
It then follows that
\[
\begin{aligned}
    \kappa_\sigma
    &\ge
    \frac{
        \Gamma(q+1)^2\Delta_q
    }{
        8\sqrt{2}\pi^{3/2}
        (q+\frac12)(q+5)^{q+\frac12}
    }
    \frac{1}{N^{q+\frac12}}\\
    &\ge
    \frac{
        \Gamma(q+1)^2\Delta_q
    }{
        2^{q+4}\pi^{3/2}
        (q+\frac12)(q+5)^{q+\frac12}
    }
    \min\left\{
        1,\,
        \left(
            \frac{\log(1/\rho)}{\log(2m)}
        \right)^{q+\frac12}
    \right\}\\
    &\ge
    \frac{
        \Gamma(q+1)^2\Delta_q
    }{
        2^{2q+\frac92}\pi^{3/2}
        (q+\frac12)(q+5)^{q+\frac12}
    }
    \frac{\theta^{2q+1}}
         {\log^{q+\frac12}(2m)}.
\end{aligned}
\]
The first inequality combines the preceding two bounds, and the second uses
\[
    N\le1+\frac{\log(2m)}{\log(1/\rho)}
    \le2\max\left\{1,\frac{\log(2m)}{\log(1/\rho)}\right\}.
\]
For the third inequality, we have \(\log(1/\rho)=\log(\sec\theta)\ge\theta^2/2\), and, for \(m\ge2, \theta\le\pi/2\), \(\theta^2/(2\log(2m))<1\). The cases \(m=1\) or \(\theta=\pi/2\) follow by taking \(N=1\): the Gershgorin radius is zero, and the tail bound implies the stated inequality because \(\theta^2<4\log(2m)\).
\end{proof}

With the above, we are now able to finish the proof of
\Cref{thm:homogeneous-stable-rank}:

\begin{proof}[Proof of \Cref{thm:homogeneous-stable-rank}]
We perform various algebraic manipulations deferred to \Cref{lem:algebra-width-stable} in Appendix \ref{app:omitted-rank}. We first substitute the lower bound
\[
    \kappa_\sigma
    \ge
    \frac{
        \Gamma(q+1)^2\Delta_q
    }{
        2^{2q+\frac92}\pi^{3/2}
        (q+\frac12)(q+5)^{q+\frac12}
    }
    \frac{\theta^{2q+1}}
         {\log^{q+\frac12}(2m)}
\]
from \Cref{prop:homogeneous-min-eigenvalue} into the sufficient condition of
\Cref{prop:homogeneous-concentration}. To use \Cref{lem:algebra-width-stable} when \(m=1\) or \(\theta>1\), we may replace \(m\) by \(\max\{m,2\}\) and \(\theta\) by \(\min\{\theta,1\}\) in the factor multiplying \(\log(m/\delta)\), leaving this confidence factor unchanged. These replacements only change the universal constants below.The resulting condition is implied by
\begin{align*}
    n \ge C'C^q \cdot \frac{M^2}{(c_+-(-1)^qc_-)^2} \cdot \frac{m}{\theta^{2q+1}} \cdot \log^{2q+\frac{1}{2}}\left(\max\left(2, \frac{M^2}{(c_+-(-1)^qc_-)^2} \right) \frac{m}{\theta}\right) \log\left(\frac{m}{\delta}\right),
\end{align*}
where \(C',C>0\) are universal constants, and we used the identity \(\Delta_q=(c_+-(-1)^qc_-)^2\) for integer \(q\). Enlarging the universal constant gives the form stated in the theorem.
\end{proof}

\begin{remark}[Implication for \(\sgn\), Heaviside step and ReLU activations]
From \Cref{thm:homogeneous-stable-rank} and the propositions used in its proof, we can derive some immediate implications for sign, Heaviside step and ReLU activation functions. Specifically, for the sign activation function, where \(q=0\), one can improve the analysis of \Cref{prop:homogeneous-concentration} to avoid truncation, since \(\|\sigma(Ug)\|_2^2=m\) almost surely, and obtain that there exist universal constants \(c,C>0\) such that
\[
    n
    \ge
    C
    \frac{m\sqrt{\log(m)}}{\theta}
    \log\left(\frac{m}{\delta}\right)
    \quad\Longrightarrow\quad
    \Pb\left(
        \lambda_{\min}(ZZ^\top)
        \ge
        c\frac{n\theta}{\sqrt{\log(m)}}
    \right)
    \ge
    1-\delta.
\]
A similar conclusion holds for the Heaviside step function where $q=0$, $c_+=1$ and $c_-=0$.
Moreover, for the ReLU activation function, where \(q=1\), there exist universal constants \(c,C>0\) such that
\[
    n
    \ge
    C
    \frac{m\log^{3/2}(m)}{\theta^3}
    \log\left(\frac{m}{\theta}\right)
    \log\left(\frac{m}{\delta}\right)
    \quad\Longrightarrow\quad
    \Pb\left(
        \lambda_{\min}(ZZ^\top)
        \ge
        c\frac{n\theta^3}{\log^{3/2}(m)}
    \right)
    \ge
    1-\delta.
\]
\end{remark}

\begin{remark}[Comparison with the ReLU derivative kernel]
The ReLU bound above concerns the \emph{random-feature covariance}
\(\E[(Ug)_+(Ug)_+^\top]\), and does not contradict the bound of
\cite{Song26}, which concerns the continuous ReLU \emph{derivative Gram
matrix}. The nonzero high-degree coefficients of that derivative kernel decay as
\(k^{-3/2}\), leading to a \(\theta/\sqrt{\log(m)}\) bound, whereas the
nonzero high-degree coefficients of the plain ReLU covariance decay as \(k^{-5/2}\), leading to
the \(\theta^3/\log^{3/2}(m)\) bound obtained here. Moreover, the present
lower-bound argument alone does not prove that the latter rate is worst-case
tight. Establishing a matching upper bound for the plain ReLU random-feature
covariance requires a separate construction; the matching construction in
\cite{Song26} is proved for the derivative Gram matrix.
\end{remark}

%% file: sections/appendix.tex
\bigskip
{\noindent \LARGE \bfseries Appendix}
\appendix

\section{Omitted Content from \Cref{sec:intro,sec:preliminaries}}\label{omitted-prelim}

\subsection{Minimum-Norm Interpolant}

\begin{lemma}[Minimum-norm interpolant]
\label{lem:min-norm-interpolant}
Let \(Z\in\mathbb R^{m\times n}\) have full row rank, and let
\(y\in\mathbb R^m\). Then the unique minimum-Euclidean-norm solution of
\(Z\bw=y\) is
\[
    \bw^\star
    =
    Z^\top(ZZ^\top)^{-1}y.
\]
Moreover,
\[
    \|\bw^\star\|_2^2
    =
    y^\top(ZZ^\top)^{-1}y
    \le
    \frac{\|y\|_2^2}{\lambda_{\min}(ZZ^\top)}.
\]
\end{lemma}

\begin{proof}
Since \(Z\) has full row rank, the matrix \(ZZ^\top\) is positive definite and therefore invertible. The vector $\bw^\star=Z^\top(ZZ^\top)^{-1}y$ interpolates \(y\), because
\[
    Z\bw^\star
    =
    ZZ^\top(ZZ^\top)^{-1}y
    =
    y.
\]
Now let \(\bw\) be any other solution of \(Z\bw=y\), so that \(Z(\bw-\bw^\star)=0\), and \(\bw-\bw^\star\in\ker(Z)\). On the other hand, \(\bw^\star\in\operatorname{range}(Z^\top)\), and
orthogonal decomposition gives $\operatorname{range}(Z^\top) = \ker(Z)^\perp$. Consequently, \(\bw^\star\) is orthogonal to
\(\bw-\bw^\star\), and hence
\[
    \|\bw\|_2^2
    =
    \|\bw^\star\|_2^2
    +
    \|\bw-\bw^\star\|_2^2
    \ge
    \|\bw^\star\|_2^2,
\]
which means that \(\bw^\star\) is the unique minimum-norm interpolant. Moreover,
\[
    \|\bw^\star\|_2^2
    =
    y^\top(ZZ^\top)^{-1}ZZ^\top(ZZ^\top)^{-1}y
    =
    y^\top(ZZ^\top)^{-1}y.
\]
Finally, since the largest eigenvalue of \((ZZ^\top)^{-1}\) is
\(1/\lambda_{\min}(ZZ^\top)\), the Rayleigh quotient bound gives
\[
    y^\top(ZZ^\top)^{-1}y
    \le
    \frac{\|y\|_2^2}{\lambda_{\min}(ZZ^\top)},
\]
as desired.
\end{proof}

\subsection{Pure Polynomial Activations}
In this section, we show that, if the activation is a pure monomial, then the
random features live in a finite-dimensional polynomial feature space. Consequently, full row rank is impossible once the number of data points exceeds this dimension.
\begin{observation}\label{obs:monomial-obstruction}
Let $\sigma(t)=ct^q$ for $c\neq0$ and $q\in\mathbb Z_{\ge0}$. For $q=0$, this means the constant activation $\sigma(t)=c$. Write $Z=\sigma(UR)$. Then $\rank(\sigma(UR)) \le \binom{d+q-1}{q}$, and so, if $m>\binom{d+q-1}{q}$, we have
\[
    \Pb(\rank(Z)=m)=0.
\]
\end{observation}

\begin{proof}
Each column of \(Z = \sigma(UR)\) is of the form $c\bigl(\langle u_1,g_j\rangle^q,\ldots,\langle u_m,g_j\rangle^q\bigr)^\top $. For fixed \(U\), the map $g\mapsto \bigl(\langle u_1,g\rangle^q,\ldots,\langle u_m,g\rangle^q\bigr)$ has coordinates in the vector space of homogeneous polynomials of degree \(q\)
in \(d\) variables, which has dimension \(\binom{d+q-1}{q}\). Writing this map in the monomial basis expresses each column as a linear combination of the corresponding coefficient vectors. Hence all columns
of \(Z\) lie in a fixed subspace of \(\R^m\) of dimension at most
\(\binom{d+q-1}{q}\), i.e.,
\[
    \rank(Z)\le \binom{d+q-1}{q}.
\]
Therefore, if \(m>\binom{d+q-1}{q}\), full row rank is impossible for every realization of
\(R\), and so we must have \(\Pb(\rank(Z)=m)=0\).
\end{proof}

\subsection{Lower Bounds on the Necessary Width}

\begin{proposition}[Confidence lower bounds]
\label{prop:exact-rank-confidence-lower-bounds}
For exact rank lifting, the following lower bounds on the width hold, where $R$ has independent standard Gaussian columns and $\delta\in(0,1)$:

\begin{enumerate}
    \item For every \(0<\theta\le\pi/2\), there are two unit vectors with
    projective separation \(\theta\) such that, for the sign activation,
    \[
        \Pb\bigl(
            \rank(\sgn(UR))<2
        \bigr)
        \ge
        \left(
            1-\frac{\theta}{\pi}
        \right)^n.
    \]
    Consequently, if the full-rank probability is at least \(1-\delta\), then
    \[
        n
        \ge
        \frac{\pi}{2\theta}
        \log\left(\frac{1}{\delta}\right).
    \]

    \item For every \(0<\theta\le\pi/4\), there are three unit vectors with
    projective separation \(\theta\) such that, for the ReLU activation,
    \[
        \Pb\bigl(
            \rank((UR)_+)<3
        \bigr)
        \ge
        \left(
            1-\frac{2\theta}{\pi}
        \right)^n.
    \]
    Consequently, if the full-rank probability is at least \(1-\delta\), then
    \[
        n
        \ge
        \frac{\pi}{4\theta}
        \log\left(\frac{1}{\delta}\right).
    \]
\end{enumerate}
\end{proposition}

\begin{proof}
For sign, choose unit vectors \(u_1,u_2 \in\mathbb R^2\) with
\(\theta_{12}=\theta\). For one Gaussian column \(g\), the
two signs agree with probability \(1-\theta/\pi\). On this event,
\(\sgn(Ug)\) belongs to the one-dimensional subspace
\(\operatorname{span}\{(1,1)^\top\}\). Therefore, if the signs agree for all
\(n\) independent columns, then the two rows of \(\sgn(UR)\) are identical,
and hence
\[
    \Pb\bigl(
        \rank(\sgn(UR))<2
    \bigr)
    \ge
    \left(
        1-\frac{\theta}{\pi}
    \right)^n.
\]
If the failure probability is at most \(\delta\), then
\[
    n
    \ge
    \frac{\log(1/\delta)}
         {\log((1-\theta/\pi)^{-1})} \ge
    \frac{\pi}{2\theta}
    \log\left(\frac{1}{\delta}\right),
\]
where the last inequality follows since \(\theta/\pi\le1/2\) and
\(\log((1-x)^{-1})\le2x\) for \(0\le x\le1/2\).

We now turn our attention to the ReLU case, and choose unit vectors $u_+=(\cos\theta,\sin\theta), u_0=(1,0), u_-=(\cos\theta,-\sin\theta) \in \R^2$. We note that their pairwise angles are \(\theta,\theta\), and \(2\theta\), so they are projectively separated by \(\theta\) when \(\theta\le\pi/4\). Moreover, $u_++u_-=2\cos\theta\,u_0$ and so, for every \(g\in\mathbb R^2\),
\begin{align}\label{eq:linear-relation-relu-lb}
    (\langle u_+,g\rangle)_+
    -2\cos\theta\,(\langle u_0,g\rangle)_+
    +(\langle u_-,g\rangle)_+
    =0
\end{align}
if and only if $\langle u_+,g\rangle\langle u_-,g\rangle\ge0$. Indeed, if
both are nonnegative, the identity follows from
\(\langle u_+,g\rangle+\langle u_-,g\rangle
=2\cos\theta\,\langle u_0,g\rangle\), and if both are nonpositive, all three positive
parts vanish. Finally, if they have opposite signs, the left-hand side equals the smaller of their two absolute values and is therefore positive.

Since the angle between \(u_+\) and \(u_-\) is \(2\theta\), their Gaussian
projections have the same sign with probability \(1-2\theta/\pi\), and, on this
event, the ReLU feature column satisfies the relation in \eqref{eq:linear-relation-relu-lb}. Hence, if this occurs for all \(n\) independent columns, then
\(\rank((UR)_+)\le2\), and therefore
\[
    \Pb\bigl(\rank((UR)_+)<3\bigr)
    \ge
    \left(1-\frac{2\theta}{\pi}\right)^n.
\]
Thus failure probability at most \(\delta\) requires
\[
    n
    \ge
    \frac{\log(1/\delta)}
         {\log((1-2\theta/\pi)^{-1})}
    \ge
    \frac{\pi}{4\theta}\log\left(\frac{1}{\delta}\right),
\]
where the last inequality uses
\(\log((1-x)^{-1})\le2x\) for \(0\le x\le1/2\).
\end{proof}

\subsection{Exact versus Stable Rank Lifting: Dependence on the Degree}
\label{app:degree-dependent-separation}

In this section we exhibit an example that separates exact and stable rank lifting as the homogeneity degree $q$ grows. For this, we consider the one-dimensional data $U=1$, set $ \sigma(t) = t_+^q$, where $q\geq 4$, let $R=(g_1,\dots , g_n)$ be independent standard Gaussian variables, and set $Z=\sigma(UR)$. Observe that the event $\rank(Z) = 1$ happens if and and only if any of the Gaussian variables attains a positive value. By independence of the $g_i$'s, we find that this happens with with probability
\begin{align*}
    \Pb\bigl(\rank(Z)=1\bigr) =1-2^{-n},
\end{align*}
independently of $q$. Stable rank lifting on the other hand strongly depends on the homogeneity degree $q$, as the following computations show. Indeed, setting $\kappa_\sigma=\E[\sigma(g)^2]$, the computation below shows that 
\begin{align}\label{eq:exp_dep_stable}
    \Pb\left(\lambda_{\min}(ZZ^\top)<\frac{n\kappa_\sigma}{4}\right) \ge \exp\left(-n e^{-q/(4e)}\right),
\end{align}
hence the necessary width $n$ for a stable rank guarantee of $1-\delta$ is exponential in $q$.
\par

\begin{proof}[Proof of \Cref{eq:exp_dep_stable}]

    Combining the Gaussian even-moment formula and Stirling's approximation yields
    \begin{align*}
        \kappa_\sigma =\E[g_+^{2q}] =\frac{(2q-1)!!}{2} \ge\frac{q!}{2} \ge\frac12\left(\frac qe\right)^q.
    \end{align*}
    
    If every $g_j\le\sqrt{q/(2e)}$, then

    \begin{align*}
        \lambda_{\min}(ZZ^\top) = ZZ^T = \sum_{j=1}^n(g_j)_+^{2q} \le n\left(\frac{q}{2e}\right)^q \le 2^{1-q}n\kappa_\sigma <\frac{n\kappa_\sigma}{4},
    \end{align*}
    
    where the last inequality uses $q\ge4$. Thus, stable rank requires at least one preactivation attaining a value larger than $\sqrt{q/(2e)}$.
    For $t\ge0$, shifting the Gaussian tail integral givesn $\Pb(g>t) \le (1/2) e^{-t^2/2}$. By independence and the inequality $\log(1-x)\ge-2x$ for
    $0\le x\le1/2$, it follows that
    \begin{align*}
        \Pb\left(\lambda_{\min}(ZZ^\top)<\frac{n\kappa_\sigma}{4}\right)
        &\ge\Pb\left(\max_{j\in[n]}g_j\le\sqrt{\frac{q}{2e}}\right) \ge\left(1-\frac12 e^{-q/(4e)}\right)^n
        \ge\exp\left(-n e^{-q/(4e)}\right),
    \end{align*}
    as was to be shown.
\end{proof}

\section{Omitted Content from \Cref{sec:exact-rank,sec:homogeneous-stable-rank}}\label{app:omitted-rank}

\subsection{Computation of the Sign-difference Bound}

\begin{lemma}[Isolated-facet crossing]\label{lem:single-crossing}
    Let $\{u_1,\dots u_k\} \subseteq\R^d$, $k\geq 2$ be unit vectors with $\theta \leq \theta_{ij} \leq \pi-\theta$ for all $i\neq j$, where $\theta \in (0, \pi/2]$. and write $U=(u_1,\ldots,u_k)^\top$. Let $g,h \sim \mathcal{N}(0, I_d)$ be independent Gaussians and let $\Tilde{g} = \cos (\varphi) g + \sin (\varphi) h$ for $\varphi = \arctan(\theta/k)$. Finally, denote by $N$ the number of differences in sign of $Ug$ and $U\Tilde{g}$, that is,
    \begin{align*}
        N = |\{i \in[k]: \sgn(\langle u_i, g \rangle) \neq \sgn( \langle u_i, \Tilde{g}\rangle)\}|.
    \end{align*}
    Then, the probability that $Ug$ and $U\Tilde{g}$ differ in exactly one position is lower-bounded by
    \begin{align*}
        \Pb{(N=1)}\ge \frac{\theta}{2\pi}.
    \end{align*}
\end{lemma}

\begin{proof}
    By the definition of $\varphi$, equivalently, $\Tilde g=(g+(\theta/k)h)/\sqrt{1+(\theta/k)^2}$. The positive denominator does not change signs, so $N$ counts the hyperplanes crossed by the segment from $g$ to $g+(\theta/k)h$, almost surely. For fixed $h$, crossing one hyperplane requires $g$ to lie in a thin strip; crossing two requires membership in an intersection of two strips. Separation controls these intersections, while the normalization ensures that $g$ and $\Tilde g$ have the same standard Gaussian distribution.

    Let $D_i$ denote the event that $\langle u_i, g\rangle$ and $\langle u_i, \Tilde{g} \rangle$ have opposite signs. By Sheppard's formula (\cite{sheppard1899application}) we know that $\Pb (D_i) = \varphi/ \pi$. In the following, we deduce an upper bound for $\Pb (D_i \cap D_j)$ for $i \neq j$ since we can then simply bound
    \begin{align} \label{eq:union_bound}
        \Pb (N=1) \geq \sum_i \Pb(D_i) - \sum_{i \neq j} \Pb (D_i \cap D_j),
    \end{align}
    where the second sum is over ordered pairs. This which follows directly from $\mathbf 1_{\{N=1\}}\ge N-N(N-1)$.
    \par
    Since $\sgn (x) \neq \sgn (y)$ is equivalent to $xy <0$  for nonzero $x,y$ and since $\tan \varphi = \theta/k$, the event ${D_i}$ can be written as $\{\langle g, u_i \rangle^2 < - \langle g,u_i \rangle \langle h, u_i \rangle \cdot \theta/k\}$. Thus, for any fixed $h$ the event $D_i$ is equivalent to the variable $\langle g, u_i\rangle$ taking a value strictly between $0$ and $-\langle u_i, h \rangle \cdot \theta/k$ and the event $\{D_i \cap D_j\}$ is equivalent to the joint random variable $(\langle g,u_i \rangle, \langle g, u_j \rangle)$ lying in a rectangle of area $|\langle u_i, h \rangle||\langle u_j, h \rangle| (\theta/k)^2$. To be more explicit, assume $z := \langle u_i, h \rangle$ and $w := \langle u_j, h \rangle$ are fixed, then the event $\{D_i \cap D_j\}$ is equivalent to $(\langle g,u_i\rangle,\langle g,u_j\rangle)$ belonging to the rectangle
   \begin{align*}
        R_{z,w}
        =
        \left(
            \min\left\{0,-\frac{\theta z}{k}\right\},
            \max\left\{0,-\frac{\theta z}{k}\right\}
        \right)
        \times
        \left(
            \min\left\{0,-\frac{\theta w}{k}\right\},
            \max\left\{0,-\frac{\theta w}{k}\right\}
        \right).
   \end{align*}
    Now, notice that the joint random variable $(\langle g,u_i \rangle, \langle g, u_j \rangle)$ is a standard bivariate normal variable with correlation $\cos \theta_{ij}$ and therefore, we can bound its joint density function everywhere from above by
    \begin{align*}
        f(x,y) = \frac{1}{2\pi \sqrt{1-\cos ^2 \theta_{ij}}} \exp\left({-\frac{x^2 -2 xy\cos \theta_{{ij}}+ y^2 }{2(1-\cos^2\theta_{ij})}}\right) \leq \frac{1}{2\pi \sqrt{1-\cos ^2 \theta_{ij}}} = \frac{1}{2\pi \sin \theta_{ij}}.
    \end{align*}
    Finally, observe that the joint random variable $(\langle h,u_i \rangle, \langle h, u_j \rangle)$ is an identically distributed and independent copy of $(\langle g,u_i \rangle, \langle g, u_j \rangle)$ so the joint density $f(x,y,z,w)$ of all four random variables $(\langle g,u_i \rangle, \langle g, u_j \rangle, \langle h,u_i \rangle, \langle h, u_j \rangle)$ factors into $f(x,y)f(z,w)$. Hence, we can apply Fubini's theorem to get
    \begin{align*}
        \Pb({D_i \cap D_j}) &= \int_{\R^2} \left( \int_{R_{z,w}} f(x,y)dxdy \right) f(z,w) dzdw \leq \int_{\R^2} \left( \int_{R_{z,w}} \frac{1}{2\pi \sin \theta_{ij}} dxdy \right) f(z,w) dzdw \\
        & =  \frac{1}{2\pi \sin \theta_{ij}}  \int_{\R^2} \frac{\theta^2}{k^2}|z||w| f(z,w) dzdw =  \frac{\theta^2}{2\pi k^2 \sin \theta_{ij}}  \E( |\langle h,u_i \rangle|| \langle h, u_j \rangle|),
    \end{align*}
    where in the penultimate step we used that $R_{z,w}$ is a rectangle of area $(\theta/k)^2|z||w|$ for every fixed $(z,w) \in \R^2$. Using Nabeya's formula for the expectation of the absolute value of a bivariate Gaussian (\cite{Nabeya1951}) we can evaluate the expectation $\E( |\langle h,u_i \rangle|| \langle h, u_j \rangle|)$ and applying the inequality $|\cot \theta| \leq 1/\min(\theta,\pi-\theta)$ for $\theta\in (0,\pi)$ yields
    \begin{align*}
        \Pb({D_i \cap D_j}) &\leq  \frac{\theta^2}{2\pi k^2 \sin \theta_{ij}}  \E( |\langle h,u_i \rangle|| \langle h, u_j \rangle|) = \frac{\theta^2}{2\pi k^2 \sin \theta_{ij}}\frac{2}{\pi} \left(\sin\theta_{ij}+ \left(\frac{\pi}{2}-\theta_{ij}\right) \cos\theta_{ij}\right)\\
        &= \frac{\theta^2}{\pi^2k^2} \left(1+\left(\frac{\pi}{2}-\theta_{ij}\right)\cot\theta_{ij} \right) \leq \frac{\theta^2}{\pi^2k^2}
        \frac{\pi}{2\min(\theta_{ij},\pi-\theta_{ij})} \leq \frac{\theta^2}{\pi^2k^2}\frac{\pi}{2\theta} = \frac{\theta}{2\pi k^2}.
    \end{align*}
    The result now follows directly from combining \eqref{eq:union_bound} with$\Pb(D_i)=\varphi/\pi$ and the inequality $\varphi=\arctan(\theta/k)\ge\theta/k-\theta^3/(3k^3)$ and the fact that $2\theta^2\leq\pi^2/2<6\leq3k$. Indeed we have,
    \begin{align*}
        \Pb(N=1) \geq\frac{k\varphi}{\pi}-k(k-1)\frac{\theta}{2\pi k^2} \geq \frac{1}{\pi} \left(\theta-\frac{\theta}{2k}\right) -\frac{(k-1)\theta}{2\pi k}
        = \frac{\theta}{2\pi},
    \end{align*}
    finishing the proof.
\end{proof}

\subsection{Concentration of the Empirical Kernel}

\prophomogeneousconcentration*

\begin{proof}
Let \(z_j=\sigma(Ug_j)\), so that
\(Z=[z_1,\ldots,z_n]\) and
\(ZZ^\top=\sum_{j=1}^n z_jz_j^\top\). Since
\(\sigma(t)=c_+t_+^q+c_-(-t)_+^q\) for $t\neq0$, we have
\[
    |\sigma(t)|
    \le
    M|t|^q
    \qquad
    \text{for every }t\in\mathbb R\setminus\{0\}.
\]
All Gaussian preactivations are nonzero almost surely, so this bound applies throughout the proof.
Next, let us fix \(L\ge1\), to be chosen below, and retain only the columns for which every
preactivation is at most \(L\) in absolute value, i.e., let the truncation be defined as
\[
    \mathcal E_j
    =
    \left\{
        \max_{i\in[m]}
        |\langle u_i,g_j\rangle|
        \le L
    \right\},
\]
with $K_L = \E[z_jz_j^\top\mathbf 1_{\mathcal E_j}]$ the corresponding truncated population kernel.
We observe that, since
\(ZZ^\top=\sum_{j=1}^n z_jz_j^\top
\succeq\sum_{j=1}^n z_jz_j^\top\mathbf 1_{\mathcal E_j}\), then we have
\begin{align}\label{eq:beta-ub}
    \Pb\left(
        \lambda_{\min}\left(
            \sum_{j=1}^n
            z_jz_j^\top\mathbf 1_{\mathcal E_j}
        \right)
        \ge \beta
    \right)
    \le
    \Pb\left(
        \lambda_{\min}\left(
            ZZ^\top
        \right)
        \ge \beta
    \right),
\end{align}
for every \(\beta\ge0\), and so it suffices to lower bound the left-hand side. Our goal is to apply matrix Chernoff to the independent positive semidefinite
matrices \(z_jz_j^\top\mathbf 1_{\mathcal E_j}\).

To do so, we first need to establish bounds on all eigenvalues of each truncated matrix. Indeed, note that, on event \(\mathcal E_j\), $\|z_j\|_2^2 = \sum_{i=1}^m \sigma(\langle u_i,g_j\rangle)^2 \le M^2mL^{2q}$, and hence
\begin{align}\label{eq:range-mtx-chernoff}
    0
    \preceq
    z_jz_j^\top\mathbf 1_{\mathcal E_j}
    \preceq
    M^2mL^{2q}I_m.
\end{align}
Next, we need to bound its expectation, and, by definition of \(K_L\), we have
\[
    \lambda_{\min}\left(
        \E\left[
            \sum_{j=1}^n
            z_jz_j^\top\mathbf 1_{\mathcal E_j}
        \right]
    \right)
    =
    n\lambda_{\min}(K_L) \ge n(\lambda_{\min}(K)-\|K-K_L\|).
\]
The last inequality holds simply because, $v^\top K_Lv = v^\top Kv-v^\top(K-K_L)v \ge \lambda_{\min}(K)-\|K-K_L\|$ for every unit vector \(v\in\mathbb R^m\), and therefore, taking the minimum over all unit vectors, $\lambda_{\min}(K_L) \ge \lambda_{\min}(K)-\|K-K_L\|$.

The remaining task is to bound \(\|K-K_L\|\) from above: since
\(K-K_L=\E[z_jz_j^\top\mathbf 1_{\bar{\mathcal E}_j}]\succeq0\), we have $\|K-K_L\| \le \E\bigl[\|z_j\|_2^2\mathbf 1_{\bar{\mathcal E}_j}\bigr]$ from Jensen's inequality. Then, using
\(\mathbf 1_{\bar{\mathcal E}_j}
\le
\sum_{\ell=1}^m
\mathbf 1_{\{|\langle u_\ell,g_j\rangle|>L\}}\),
we obtain
\begin{align*}
    \|K-K_L\|
    \le
    M^2
    \sum_{i,\ell=1}^m
    \E\left[
        |\langle u_i,g_j\rangle|^{2q}
        \mathbf 1_{\{|\langle u_\ell,g_j\rangle|>L\}}
    \right].
\end{align*}
For fixed \(i,\ell\), both
\(\langle u_i,g_j\rangle\) and
\(\langle u_\ell,g_j\rangle\) are standard Gaussian random variables, since
\(g_j\sim \mathcal N(0,I_d)\) and
\(\|u_i\|_2=\|u_\ell\|_2=1\). Thus, letting \(\zeta\sim N(0,1)\), Cauchy--Schwarz yields
\begin{align*}
    \E\left[
        |\langle u_i,g_j\rangle|^{2q}
        \mathbf 1_{\{|\langle u_\ell,g_j\rangle|>L\}}
    \right] \le
    \E[|\zeta|^{4q}]^{1/2}
    \Pb(|\zeta|>L)^{1/2} \le
    2^{q+\frac12}
    \left(
        \frac{\Gamma(2q+\frac12)}{\sqrt{\pi}}
    \right)^{1/2}
    e^{-L^2/4},
\end{align*}
where we used
\(\E|\zeta|^{4q}=2^{2q}\Gamma(2q+\frac12)/\sqrt{\pi}\) and \(\Pb(|\zeta|>L)\le2e^{-L^2/2}\). Observe also that the Gaussian random variables correlation is irrelevant, because Cauchy--Schwarz does not require independence. Thus,
\[
    \|K-K_L\|
    \le m^2 M^2 \cdot
    2^{q+\frac12}
    \left(
        \frac{\Gamma(2q+\frac12)}{\sqrt{\pi}}
    \right)^{1/2} e^{-L^2/4} = \frac{\kappa_\sigma}{2},
\]
where last equality follows by choosing
\[
    L^2
    =
    4\log\left(
        \frac{
            2^{q+\frac32}M^2m^2
        }{
            \kappa_\sigma
        }
        \left(
            \frac{\Gamma(2q+\frac12)}{\sqrt{\pi}}
        \right)^{1/2}
    \right).
\]
This choice satisfies $L\ge1$: indeed,
$\kappa_\sigma\le\E[\sigma(\zeta)^2]\le M^2\E|\zeta|^{2q}\le M^2\E[|\zeta|^{4q}]^{1/2}$,
so the argument of the logarithm is at least $2\sqrt2\,m^2$. Summarizing, we conclude that
\[
    \lambda_{\min}\left(
        \E\left[
            \sum_{j=1}^n
            z_jz_j^\top\mathbf 1_{\mathcal E_j}
        \right]
    \right)
    =
    n\lambda_{\min}(K_L)
    \ge
    \frac{n\kappa_\sigma}{2}.
\]
We finally apply matrix Chernoff
(see \cite[Theorem 5.1.1]{Tropp15}), with lower-tail parameter \(1/2\),
\(\beta=n\kappa_\sigma/4\) in \eqref{eq:beta-ub}, and the range specified in
\eqref{eq:range-mtx-chernoff}:
\begin{align*}
    \Pb\left(
        \lambda_{\min}(ZZ^\top)
        <
        \frac{n\kappa_\sigma}{4}
    \right) \le
    \Pb\left(
        \lambda_{\min}\left(
            \sum_{j=1}^n
            z_jz_j^\top\mathbf 1_{\mathcal E_j}
        \right)
        <
        \frac{n\kappa_\sigma}{4}
    \right) &\le
    m\exp\left(
        -\frac{
            (1-\log(2))n\kappa_\sigma
        }{
            4M^2mL^{2q}
        }
    \right).
\end{align*}
By the choice of $L$, we have that taking
\[
     n
    \ge
    \frac{4^{q+1}M^2}{1-\log(2)}
    \frac{m}{\kappa_\sigma}\cdot
    \log^q\left(
        \frac{
            2^{q+\frac32}M^2m^2
        }{
            \kappa_\sigma
        }
        \left(
            \frac{\Gamma(2q+\frac12)}{\sqrt{\pi}}
        \right)^{1/2}
    \right)
    \log\left(\frac{m}{\delta}\right)
\]
makes the failure probability at most \(\delta\).
\end{proof}

\subsection{Taylor Coefficients of $q$-Homogeneous Activations and Their Tail}

\lemhomogeneoustaylor*
\begin{proof}
Let
\(\varphi(x)=e^{-x^2/2}/\sqrt{2\pi}\)
denote the standard Gaussian density, and let \(p_t(x,y)\) denote the density
of a standard bivariate Gaussian pair with correlation \(t\), for $|t|<1$. Thus,
\[
    p_t(x,y)
    =
    \frac{1}{2\pi\sqrt{1-t^2}}
    \exp\left(
        -\frac{x^2+y^2-2txy}{2(1-t^2)}
    \right).
\]
A direct differentiation gives
\[
    \partial_t p_t(x,y)
    =
    \partial_x\partial_y p_t(x,y)
    =
    \left[
        \frac{t}{1-t^2}
        +
        \frac{(x-ty)(y-tx)}{(1-t^2)^2}
    \right]
    p_t(x,y).
\]
Iterating this identity and evaluating at \(t=0\), where
\(p_0(x,y)=\varphi(x)\varphi(y)\), yields
\[
    \left.
        \partial_t^k p_t(x,y)
    \right|_{t=0}
    =
    \varphi^{(k)}(x)\varphi^{(k)}(y).
\]
For every fixed \(0<r<1\) and derivative order, the derivatives of \(p_t\), for \(|t|\le r\), are
bounded by a polynomial in \(|x|+|y|\) times
\(C_re^{-c_r(x^2+y^2)}\). Since \(\sigma\) has polynomial growth, this
justifies differentiating under the integral.
This identifies the Taylor coefficients. Convergence to $\psi_\sigma$ for $|t|<1$ follows from the Gaussian Hermite expansion \cite{DanielyFrostigSinger16,HolzmullerScholpple26}: the functions $(-1)^k\varphi^{(k)}/(\sqrt{k!}\,\varphi)$ form the orthonormal Gaussian Hermite basis, and the correlation kernel of any square-integrable activation has the squares of its Hermite coefficients as its power-series coefficients. Since $\sigma$ has polynomial growth, it is square-integrable under the Gaussian measure.
Consequently,
\[
    \psi_\sigma(t)
    =
    \sum_{k=0}^\infty
    \frac{t^k}{k!}
    \left(
        \int_{\mathbb R}
        \sigma(x)\varphi^{(k)}(x)\,dx
    \right)^2,
\]
and in particular every Taylor coefficient is nonnegative.

It remains to compute the one-dimensional integral. Since
\(\varphi^{(k)}(-x)=(-1)^k\varphi^{(k)}(x)\), the two branches of
\(\sigma\) give
\begin{align}\label{eq:sigma-branches}
    \int_{\mathbb R}
    \sigma(x)\varphi^{(k)}(x)\,dx
    =
    \bigl(c_++(-1)^kc_-\bigr)
    \int_0^\infty
    x^q\varphi^{(k)}(x)\,dx.
\end{align}
For this integral calculation, temporarily allow $q$ to be real. For \(q>k-1\), we first recall the duplication identity
\[
    \Gamma(q-k+1)
    =
    \frac{2^{q-k}}{\sqrt{\pi}}
    \Gamma\left(\frac{q-k+1}{2}\right)
    \Gamma\left(\frac{q-k+2}{2}\right).
\]
Then, integrating by parts \(k\) times and using the Gaussian
moment formula gives
\begin{align}\label{eq:integration-by-parts-gaussian}
    \int_0^\infty
    x^q\varphi^{(k)}(x)\,dx
    &=
    (-1)^k
    \frac{\Gamma(q+1)}{\Gamma(q-k+1)}
    \int_0^\infty
    x^{q-k}\varphi(x)\,dx \notag\\
    &=
    (-1)^k
    \frac{\Gamma(q+1)}{\Gamma(q-k+1)}
    \frac{
        2^{\frac{q-k}{2}-1}
        \Gamma\left(\frac{q-k+1}{2}\right)
    }{\sqrt{\pi}} \notag\\
    &=
    (-1)^k
    \frac{
        2^{\frac{k-q}{2}-1}\Gamma(q+1)
    }{
        \Gamma\left(\frac{q-k+2}{2}\right)
    },
\end{align}
where the last equality follows from the duplication identity above.
The integral and the final expression in \eqref{eq:integration-by-parts-gaussian} are analytic functions of \(q\) on the half-plane
\(\operatorname{Re}(q)>-1\), and since they agree for \(q>k-1\), the identity
theorem extends this formula to every \(q\ge0\). Here $1/\Gamma$ denotes its entire extension, with value zero at nonpositive integers. Substituting
\eqref{eq:integration-by-parts-gaussian} into
\eqref{eq:sigma-branches}, squaring it, and dividing by \(k!\) proves the claimed formula for \(a_k\).

We finally estimate the high-degree Taylor tail. For \(k>q\), the reflection
and duplication formulas give
\begin{align}\label{eq:explicit-ak-reflection}
    a_k
    &=
    \frac{\Gamma(q+1)^2}{2^{q+2}\pi^{3/2}}
    \bigl(c_++(-1)^kc_-\bigr)^2
    \sin^2\left(\frac{\pi(k-q)}{2}\right) \cdot
    \frac{
        \Gamma\left(\frac{k-q}{2}\right)^2
    }{
        \Gamma\left(\frac{k+1}{2}\right)
        \Gamma\left(\frac{k+2}{2}\right)
    }\notag\\
    &\ge
    \frac{\Gamma(q+1)^2}{\sqrt{2}\pi^{3/2}}
    \frac{
        \bigl(c_++(-1)^kc_-\bigr)^2
        \sin^2\left(\frac{\pi(k-q)}{2}\right)
    }{
        (k+2)^{q+\frac32}
    }.
\end{align}
In the inequality above, we use the elementary Gamma-ratio bound $\Gamma(x+\gamma)/\Gamma(x) \le (x+\gamma)^\gamma$ for $x>0, \gamma\ge0$ with \(x=(k-q)/2\), \(\gamma=(q+1)/2\), and \(\gamma=(q+2)/2\), respectively:
\[
    \frac{
        \Gamma\left(\frac{k-q}{2}\right)^2
    }{
        \Gamma\left(\frac{k+1}{2}\right)
        \Gamma\left(\frac{k+2}{2}\right)
    }
    \ge
    \frac{
        2^{q+\frac32}
    }{
        (k+1)^{\frac{q+1}{2}}
        (k+2)^{\frac{q+2}{2}}
    }
    \ge
    \frac{
        2^{q+\frac32}
    }{
        (k+2)^{q+\frac32}
    }.
\]
Now, for even \(k\), the numerator depending on the parity equals $(c_++c_-)^2 \sin^2\left(\pi q/2\right)$, whereas for odd \(k\) it equals $(c_+-c_-)^2 \cos^2\left(\pi q/2\right)$. Their sum is \(\Delta_q\), so one of the two parity classes has numerator at
least \(\Delta_q/2\).

Let \(k_0\) be the first integer of this parity satisfying
\(k_0\ge\max\{N,\lfloor q\rfloor+1\}\). Then \(k_0>q\),
\(k_0\le\max\{N,\lfloor q\rfloor+1\}+1\), and hence
\(k_0+2\le N+q+4\le(q+5)N\). Since
\(x\mapsto x^{-q-3/2}\) is decreasing,
\begin{align*}
    \sum_{r=0}^\infty
    (k_0+2r+2)^{-q-\frac32}
    \ge
    \frac12
    \int_{k_0+2}^\infty
    x^{-q-\frac32}\,dx =
    \frac{
        (k_0+2)^{-q-\frac12}
    }{
        2(q+\frac12)
    }.
\end{align*}
It follows that
\[
    \sum_{k=N}^\infty a_k
    \ge
    \frac{
        \Gamma(q+1)^2\Delta_q
    }{
        4\sqrt{2}\pi^{3/2}
        (q+\frac12)(q+5)^{q+\frac12}
    }
    \frac{1}{N^{q+\frac12}}.
\]
All that remains is to extend the expansion to the endpoints. To this end, let \(\xi,\eta\) be independent standard Gaussian random variables and note that, as \(t\uparrow1\), $\bigl(\xi,t\xi+\sqrt{1-t^2}\eta\bigr) \longrightarrow (\xi,\xi)$ almost surely, and $\sigma$ is continuous away from zero. Polynomial moment bounds then justify dominated convergence. Thus \(\psi_\sigma(t)\to\E[\sigma(\xi)^2]\). Since \(a_k\ge0\), monotone
convergence gives
\[
    \sum_{k=0}^\infty a_k
    =
    \E[\sigma(\xi)^2]
    =
    (c_+^2+c_-^2)
    \frac{
        2^{q-1}\Gamma(q+\frac12)
    }{
        \sqrt{\pi}
    }.
\]
The series therefore converges absolutely at both \(t=1\) and \(t=-1\), and
the case \(t=-1\) follows similarly by letting \(t\downarrow-1\).
\end{proof}

\begin{lemma}\label{lem:algebra-width-stable}
Let $m \ge 2$, $\theta \in (0, 1]$, $\delta\in(0,1)$, and $q\in\mathbb Z_{\ge0}$. Then, the expression
\begin{equation*}
\begin{aligned}
    &
    \frac{
        2^{4q+\frac{13}{2}}\pi^{3/2}
        (q+\frac12)(q+5)^{q+\frac12}M^2
    }{
        (1-\log(2))\Gamma(q+1)^2\Delta_q
    }
    \frac{
        m\log^{q+\frac12}(2m)
    }{
        \theta^{2q+1}
    }\\
    &\cdot
    \log^q\left(
        \frac{
            2^{3q+6}\pi^{3/2}
            (q+\frac12)(q+5)^{q+\frac12}M^2
        }{
            \Gamma(q+1)^2\Delta_q
        }
        \left(
            \frac{\Gamma(2q+\frac12)}{\sqrt{\pi}}
        \right)^{1/2}
        \frac{
            m^2\log^{q+\frac12}(2m)
        }{
            \theta^{2q+1}
        }
    \right)
    \log\left(\frac{m}{\delta}\right)
\end{aligned}
\end{equation*}
is upper bounded by:
\[
    10^6 \cdot (23000)^q \cdot \frac{M^2}{\Delta_q} \cdot \frac{m}{\theta^{2q+1}} \log^{2q+\frac{1}{2}}\left(\max\left(2, \frac{M^2}{\Delta_q} \right) \frac{m}{\theta}\right) \log\left(\frac{m}{\delta}\right).
\]
\end{lemma}

\begin{proof}
We let $A$ denote the large argument inside the $\log^q(A)$ term for the sake of brevity. We use Stirling bounds to control the dependence on $q$. Namely, we begin with the case $q \ge 1$ and use the lower bound $\Gamma(q+1)^2 \ge 2\pi q^{2q+1} e^{-2q}$ as well as the loose upper bound $(q+5)^{q+\frac{1}{2}} \le (6q)^{q+\frac{1}{2}}$, so that the $q$-dependent terms in the outer multiplier scale as:
\[
    \frac{(q+5)^{q+\frac{1}{2}}}{\Gamma(q+1)^2} \le \frac{(6q)^{q+\frac{1}{2}}}{2\pi q^{2q+1} e^{-2q}} = \frac{6^{q+\frac12}e^{2q}}{2\pi q^{q+\frac12}} = \frac{\sqrt{6}}{2\pi}(6e^2)^q q^{-(q+\frac{1}{2})}.
\]
Next, we evaluate the argument $A$. In addition to the identical $\Gamma(q+1)^{-2}$ decay in its denominator, its numerator contains the factor $\Gamma(2q+1/2)^{1/2}$. Applying the loose upper bound $\Gamma(2q+1/2) \le \Gamma(2q+1) \le e\sqrt{2\pi}(2q)^{2q+1/2}e^{-2q}$ (valid for $q \ge 1$) and taking the square root yields:
\[
    \Gamma\left(2q+\frac{1}{2}\right)^{1/2} \le (e\sqrt{2\pi})^{1/2} 2^{1/4} q^{1/4} \left(\frac{2}{e}\right)^q q^q \le 4 q^{1/4} \left(\frac{2}{e}\right)^q q^q.
\]
Substituting this alongside our previous bound for $(q+5)^{q+1/2}/\Gamma(q+1)^2$ into $A$, we obtain:
\begin{align*}
    A &\le 2^{3q+6}\pi^{5/4} \left(q+\frac{1}{2}\right) \cdot \frac{\sqrt{6}}{2\pi}(6e^2)^q q^{-q-\frac{1}{2}} \cdot 4 q^{1/4} \left(\frac{2}{e}\right)^q q^q \cdot \frac{M^2}{\Delta_q} \cdot \frac{m^2\log^{q+\frac{1}{2}}(2m)}{\theta^{2q+1}}\\
    &\le 500 \left(96 e^3\right)^q \frac{M^2}{\Delta_q} \cdot \frac{m^2 \log^{q+\frac{1}{2}}(2m)}{\theta^{2q+1}}.
\end{align*}
For the last inequality above, we note that, when multiplied together, the $q^q$ growth from the numerator perfectly neutralizes the $q^{-q}$ factor from the denominator. Applying $q+1/2 \le 1.5q$ and noting that $q \cdot q^{-1/2} \cdot q^{1/4} = q^{3/4} \le e^q$, the remaining numerical prefactor is $192\sqrt6\,\pi^{1/4}<500e$. Since $q\ge1$, this gives the displayed bound. Taking the logarithm of $A$, we have:
\begin{align*}
    &\log(A)\\
    \le~ &q \log(48000e^3) + \log\left(\frac{M^2}{\Delta_q} \right) + 2\log(m) + (2q+1)\log\left(\frac{1}{\theta}\right) + \left(q+\frac{1}{2}\right)\log(\log(2m)) \\
    \le~ & 16 q \log\left(\max\left(2, \frac{M^2}{\Delta_q} \right)\frac{m}{\theta}\right).
\end{align*}
The last inequality above holds since, for $m \ge 2$ and $\theta \in (0, 1]$, the logarithm on the right is at least $\log4$, and $\log(48000e^3)<10\log4$. The remaining terms satisfy:
\begin{itemize}
    \item $\log(M^2/\Delta_q) \le q \log\left(\max(2, M^2/\Delta_q)\cdot m/\theta\right)$;
    \item $2\log(m) + (2q+1)\log(1/\theta) \le (2q+1)\log(m/\theta) \le 3q\log\left(m/\theta\right)$;
    \item $(q+1/2)\log(\log(2m))\le1.5q\log(m/\theta)$, since $\log(2m)\le m/\theta$.
\end{itemize}
The argument $A$ is larger than $1$: log-convexity of $\Gamma$ gives $\Gamma(2q+1/2)^{1/2}\ge\Gamma(q+1)$, and $\Gamma(q+1)=q!\le q^q$ and $\Delta_q\le4M^2$ then give this directly from its definition.
To conclude the case $q \ge 1$, we multiply the bounded outer expression by $\log^q(A)$. Observing that $2^{13/2} \pi^{3/2} \cdot 1.5 \sqrt{6}/((1-\log 2) \cdot 2\pi) \le 961$, a factor $\sqrt q$ remains after cancellation of $q^q$. Dividing out the common positive factor $\log(m/\delta)$, the expression is bounded by
\begin{align*}
    &961 \cdot 16^q (6e^2)^q q^{-q+\frac{1}{2}} \frac{M^2}{\Delta_q} \cdot \frac{m \log^{q+\frac{1}{2}}(2m)}{\theta^{2q+1}} \cdot 16^q q^q \log^q\left(\max\left(2, \frac{M^2}{\Delta_q} \right)\frac{m}{\theta}\right) \\
    =~ &961 \cdot q^{\frac{1}{2}} (1536 e^2)^q \frac{M^2}{\Delta_q} \cdot \frac{m}{\theta^{2q+1}} \log^{q+\frac{1}{2}}(2m) \log^q\left(\max\left(2, \frac{M^2}{\Delta_q} \right)\frac{m}{\theta}\right) \\
    \le~ &961 \cdot (23000)^q \frac{M^2}{\Delta_q} \cdot \frac{m}{\theta^{2q+1}} \log^{2q+\frac{1}{2}}\left(\max\left(2, \frac{M^2}{\Delta_q} \right)\frac{m}{\theta}\right).
\end{align*}
Here we used $\sqrt q\le2^q$, $3072e^2<23000$, and
$\log(2m)\le\log(\max(2,M^2/\Delta_q)m/\theta)$.
Restoring the factor $\log(m/\delta)$ and using $961<10^6$ proves the claim for $q\ge1$.

Finally, since $q$ is an integer, only $q=0$ remains. In this case the factor $\log^q(A)$ equals $1$, and the expression in the lemma is exactly
\[
    \frac{2^{11/2}\pi^{3/2}\sqrt5}{1-\log2}
    \frac{M^2}{\Delta_0}
    \frac{m\sqrt{\log(2m)}}{\theta}
    \log\left(\frac{m}{\delta}\right).
\]
Note that the numerical factor is less than $2000<10^6$, and
$\log(2m)\le\log(\max(2,M^2/\Delta_0)m/\theta)$, proving the stated bound also for $q=0$.
\end{proof}

%% file: main.bbl
\newcommand{\etalchar}[1]{$^{#1}$}
\begin{thebibliography}{KMTM24}

\bibitem[AKM{\etalchar{+}}17]{AvronEtAl17}
Haim Avron, Michael Kapralov, Cameron Musco, Christopher Musco, Ameya Velingker, and Amir Zandieh.
\newblock Random fourier features for kernel ridge regression: Approximation bounds and statistical guarantees.
\newblock In {\em {ICML}}, volume~70 of {\em Proceedings of Machine Learning Research}, pages 253--262. {PMLR}, 2017.

\bibitem[ALS19]{AllenZhuLiSong19}
Zeyuan Allen{-}Zhu, Yuanzhi Li, and Zhao Song.
\newblock A convergence theory for deep learning via over-parameterization.
\newblock In {\em {ICML}}, volume~97 of {\em Proceedings of Machine Learning Research}, pages 242--252. {PMLR}, 2019.

\bibitem[Bac17]{Bach17}
Francis~R. Bach.
\newblock On the equivalence between kernel quadrature rules and random feature expansions.
\newblock {\em J. Mach. Learn. Res.}, 18:21:1--21:38, 2017.

\bibitem[BB21]{BiettiBach21}
Alberto Bietti and Francis~R. Bach.
\newblock Deep equals shallow for relu networks in kernel regimes.
\newblock In {\em {ICLR}}. OpenReview.net, 2021.

\bibitem[BELM20]{bubeck2020networksizeweightssize}
S{\'{e}}bastien Bubeck, Ronen Eldan, Yin~Tat Lee, and Dan Mikulincer.
\newblock Network size and size of the weights in memorization with two-layers neural networks.
\newblock In {\em NeurIPS}, 2020.

\bibitem[BHMM19]{BelkinEtAl19}
Mikhail Belkin, Daniel Hsu, Siyuan Ma, and Soumik Mandal.
\newblock Reconciling modern machine-learning practice and the classical bias–variance trade-off.
\newblock {\em Proceedings of the National Academy of Sciences}, 116(32):15849--15854, 2019.

\bibitem[BV19]{BaldiVershynin19}
Pierre Baldi and Roman Vershynin.
\newblock The capacity of feedforward neural networks.
\newblock {\em Neural Networks}, 116:288--311, 2019.

\bibitem[CCMO25]{CarvalhoCostaMouraoOliveira25}
Lu{\'{\i}}s Carvalho, Jo{\~{a}}o~Lopes Costa, Jos{\'{e}} Mour{\~{a}}o, and Gon{\c{c}}alo Oliveira.
\newblock The positivity of the neural tangent kernel.
\newblock {\em {SIAM} J. Math. Data Sci.}, 7(2):495--515, 2025.

\bibitem[Cov65]{Cover65}
Thomas~M. Cover.
\newblock Geometrical and statistical properties of systems of linear inequalities with applications in pattern recognition.
\newblock {\em {IEEE} Trans. Electron. Comput.}, 14(3):326--334, 1965.

\bibitem[CS09]{ChoSaul09}
Youngmin Cho and Lawrence~K. Saul.
\newblock Kernel methods for deep learning.
\newblock In {\em {NeurIPS}}, pages 342--350. Curran Associates, Inc., 2009.

\bibitem[CX21]{ChenXu21}
Lin Chen and Sheng Xu.
\newblock Deep neural tangent kernel and laplace kernel have the same {RKHS}.
\newblock In {\em {ICLR}}. OpenReview.net, 2021.

\bibitem[Dan20]{Daniely20}
Amit Daniely.
\newblock Neural networks learning and memorization with (almost) no over-parameterization.
\newblock In {\em NeurIPS}, 2020.

\bibitem[DBC{\etalchar{+}}26]{DragoBCMSSB26}
Andrea Drago, Maria~Sofia Bucarelli, Francesco Caso, Marius Michetti, Federico Siciliano, Fabrizio Silvestri, and Luca Becchetti.
\newblock Rank lifting and random non-linear maps.
\newblock In {\em The 29th International Conference on Artificial Intelligence and Statistics}, 2026.

\bibitem[DFS16]{DanielyFrostigSinger16}
Amit Daniely, Roy Frostig, and Yoram Singer.
\newblock Toward deeper understanding of neural networks: The power of initialization and a dual view on expressivity.
\newblock In {\em {NeurIPS}}, pages 2253--2261, 2016.

\bibitem[DGJS22]{dirksen2022separation}
Sjoerd Dirksen, Martin Genzel, Laurent Jacques, and Alexander Stollenwerk.
\newblock The separation capacity of random neural networks.
\newblock {\em J. Mach. Learn. Res.}, 23:309:1--309:47, 2022.

\bibitem[DPC{\etalchar{+}}25]{pmlr-v258-dandi25a}
Yatin Dandi, Luca Pesce, Hugo Cui, Florent Krzakala, Yue~M. Lu, and Bruno Loureiro.
\newblock A random matrix theory perspective on the spectrum of learned features and asymptotic generalization capabilities.
\newblock In {\em {AISTATS}}, volume 258 of {\em Proceedings of Machine Learning Research}, pages 2224--2232. {PMLR}, 2025.

\bibitem[DZPS19]{DuZhaiPoczosSingh19}
Simon~S. Du, Xiyu Zhai, Barnab{\'{a}}s P{\'{o}}czos, and Aarti Singh.
\newblock Gradient descent provably optimizes over-parameterized neural networks.
\newblock In {\em {ICLR} (Poster)}. OpenReview.net, 2019.

\bibitem[FGS25]{FroeseGrilloSkutella25}
Vincent Froese, Moritz Grillo, and Martin Skutella.
\newblock Complexity of injectivity and verification of relu neural networks (extended abstract).
\newblock In {\em {COLT}}, volume 291 of {\em Proceedings of Machine Learning Research}, pages 2188--2189. {PMLR}, 2025.

\bibitem[FSX{\etalchar{+}}26]{RaghuEtAl17}
Qin Fang, Lei Shi, Min Xu, Ding{-}Xuan Zhou, and Qi{-}Hang Zhou.
\newblock On the expressive power of deep 2d relu convolutional neural networks.
\newblock {\em J. Approx. Theory}, 320:106332, 2026.

\bibitem[FW20]{FanWang20}
Zhou Fan and Zhichao Wang.
\newblock Spectra of the conjugate kernel and neural tangent kernel for linear-width neural networks.
\newblock In {\em NeurIPS}, 2020.

\bibitem[GMS22]{ghosal2022randomly}
Promit Ghosal, Srinath Mahankali, and Yihang Sun.
\newblock Randomly initialized one-layer neural networks make data linearly separable.
\newblock {\em CoRR}, abs/2205.11716, 2022.

\bibitem[HLXZ20]{HeLiXuZheng20}
Juncai He, Lin Li, Jinchao Xu, and Chunyue Zheng.
\newblock Relu deep neural networks and linear finite elements.
\newblock {\em Journal of Computational Mathematics}, 38(3):502--527, 2020.

\bibitem[HS26]{HolzmullerScholpple26}
David Holzm{\"u}ller and Max Sch{\"o}lpple.
\newblock Beyond {ReLU}: How activations affect neural kernels and random wide networks.
\newblock In {\em Proceedings of the 29th International Conference on Artificial Intelligence and Statistics}, 2026.

\bibitem[HZS06]{HuangZhuSiew06}
Guang{-}Bin Huang, Qin{-}Yu Zhu, and Chee~Kheong Siew.
\newblock Extreme learning machine: Theory and applications.
\newblock {\em Neurocomputing}, 70(1-3):489--501, 2006.

\bibitem[IP95]{IgelnikPao95}
Boris Igelnik and Yoh{-}Han Pao.
\newblock Stochastic choice of basis functions in adaptive function approximation and the functional-link net.
\newblock {\em {IEEE} Trans. Neural Networks}, 6(6):1320--1329, 1995.

\bibitem[Jae01]{Jaeger01}
Herbert Jaeger.
\newblock The''echo state''approach to analysing and training recurrent neural networks.
\newblock 2001.

\bibitem[JHG18]{JacotGabrielHongler18}
Arthur Jacot, Cl{\'{e}}ment Hongler, and Franck Gabriel.
\newblock Neural tangent kernel: Convergence and generalization in neural networks.
\newblock In {\em NeurIPS}, pages 8580--8589, 2018.

\bibitem[KMM24]{KarhadkarMurrayMontufar24}
Kedar Karhadkar, Michael Murray, and Guido~F. Mont{\'{u}}far.
\newblock Bounds for the smallest eigenvalue of the {NTK} for arbitrary spherical data of arbitrary dimension.
\newblock In {\em NeurIPS}, 2024.

\bibitem[KMTM24]{karhadkar2023mildly}
Kedar Karhadkar, Michael Murray, Hanna Tseran, and Guido Mont{\'{u}}far.
\newblock Mildly overparameterized relu networks have a favorable loss landscape.
\newblock {\em Trans. Mach. Learn. Res.}, 2024, 2024.

\bibitem[LHCS22]{LiuEtAl22}
Fanghui Liu, Xiaolin Huang, Yudong Chen, and Johan A.~K. Suykens.
\newblock Random features for kernel approximation: {A} survey on algorithms, theory, and beyond.
\newblock {\em {IEEE} Trans. Pattern Anal. Mach. Intell.}, 44(10):7128--7148, 2022.

\bibitem[LJ09]{LukoseviciusJaeger09}
Mantas Lukosevicius and Herbert Jaeger.
\newblock Reservoir computing approaches to recurrent neural network training.
\newblock {\em Comput. Sci. Rev.}, 3(3):127--149, 2009.

\bibitem[LLC18]{LouartLiaoCouillet18}
Cosme Louart, Zhenyu Liao, and Romain Couillet.
\newblock {A random matrix approach to neural networks}.
\newblock {\em The Annals of Applied Probability}, 28(2):1190 -- 1248, 2018.

\bibitem[LLPS93]{LeshnoEtAl93}
Moshe Leshno, Vladimir~Ya. Lin, Allan Pinkus, and Shimon Schocken.
\newblock Multilayer feedforward networks with a nonpolynomial activation function can approximate any function.
\newblock {\em Neural Networks}, 6(6):861--867, 1993.

\bibitem[LMX25]{LiuMX25}
Xinliang Liu, Tong Mao, and Jinchao Xu.
\newblock Condition numbers and eigenvalue spectra of shallow networks on spheres.
\newblock {\em CoRR}, abs/2511.02625, 2025.

\bibitem[LR20]{LiangRakhlin20}
Tengyuan Liang and Alexander Rakhlin.
\newblock {Just interpolate: Kernel “Ridgeless” regression can generalize}.
\newblock {\em The Annals of Statistics}, 48(3):1329 -- 1347, 2020.

\bibitem[LSS{\etalchar{+}}14]{LopezPazEtAl14}
David Lopez{-}Paz, Suvrit Sra, Alexander~J. Smola, Zoubin Ghahramani, and Bernhard Sch{\"{o}}lkopf.
\newblock Randomized nonlinear component analysis.
\newblock In {\em {ICML}}, volume~32 of {\em {JMLR} Workshop and Conference Proceedings}, pages 1359--1367. JMLR.org, 2014.

\bibitem[Mis22]{misiakiewicz2022spectruminnerproductkernelmatrices}
Theodor Misiakiewicz.
\newblock Spectrum of inner-product kernel matrices in the polynomial regime and multiple descent phenomenon in kernel ridge regression, 2022.

\bibitem[MJBM23]{MurrayJinBowmanMontufar23}
Michael Murray, Hui Jin, Benjamin Bowman, and Guido Mont{\'{u}}far.
\newblock Characterizing the spectrum of the {NTK} via a power series expansion.
\newblock In {\em {ICLR}}. OpenReview.net, 2023.

\bibitem[MM19]{MeiMontanari22}
Song Mei and Andrea Montanari.
\newblock The generalization error of random features regression: Precise asymptotics and the double descent curve.
\newblock {\em Communications on Pure and Applied Mathematics}, 75, 2019.

\bibitem[MNM02]{MaassNatschlaegerMarkram02}
Wolfgang Maass, Thomas Natschl{\"a}ger, and Henry Markram.
\newblock Real-time computing without stable states: A new framework for neural computation based on perturbations.
\newblock {\em Neural Computation}, 14(11):2531--2560, 2002.

\bibitem[MPCB14]{MontufarEtAl14}
Guido~F. Mont{\'u}far, R{\u{a}}zvan Pascanu, Kyunghyun Cho, and Yoshua Bengio.
\newblock On the number of linear regions of deep neural networks.
\newblock In {\em Advances in Neural Information Processing Systems 27}, pages 2924--2932, 2014.

\bibitem[MZ22]{MontanariZhong22}
Andrea Montanari and Yiqiao Zhong.
\newblock The interpolation phase transition in neural networks: Memorization and generalization under lazy training.
\newblock {\em The Annals of Statistics}, 50(5):2816--2847, 2022.

\bibitem[Nab51]{Nabeya1951}
Seiji Nabeya.
\newblock Absolute moments in 2-dimensional normal distribution.
\newblock {\em Annals of the Institute of Statistical Mathematics}, 3:1, 1951.

\bibitem[Nea96]{Neal96}
Radford~M. Neal.
\newblock {\em Bayesian Learning for Neural Networks}, volume 118 of {\em Lecture Notes in Statistics}.
\newblock Springer, New York, 1996.

\bibitem[NMM21]{NguyenMondelliMontufar21}
Quynh Nguyen, Marco Mondelli, and Guido~F. Mont{\'{u}}far.
\newblock Tight bounds on the smallest eigenvalue of the neural tangent kernel for deep relu networks.
\newblock In {\em {ICML}}, volume 139 of {\em Proceedings of Machine Learning Research}, pages 8119--8129. {PMLR}, 2021.

\bibitem[OS20]{OymakSoltanolkotabi20}
Samet Oymak and Mahdi Soltanolkotabi.
\newblock Toward moderate overparameterization: Global convergence guarantees for training shallow neural networks.
\newblock {\em {IEEE} J. Sel. Areas Inf. Theory}, 1(1):84--105, 2020.

\bibitem[Pin99]{Pinkus99}
Allan Pinkus.
\newblock Approximation theory of the mlp model in neural networks.
\newblock {\em Acta Numerica}, 8:143–195, 1999.

\bibitem[PSG20]{PanigrahiShettyGoyal20}
Abhishek Panigrahi, Abhishek Shetty, and Navin Goyal.
\newblock Effect of activation functions on the training of overparametrized neural nets.
\newblock In {\em {ICLR}}. OpenReview.net, 2020.

\bibitem[PTS20]{PetzkaTrimmelSminchisescu20}
Henning Petzka, Martin Trimmel, and Cristian Sminchisescu.
\newblock Notes on the symmetries of 2-layer relu-networks.
\newblock In {\em {NLDL}}, pages 1--6. Septentrio Academic Publishing, 2020.

\bibitem[PW17]{PenningtonWorah17}
Jeffrey Pennington and Pratik Worah.
\newblock Nonlinear random matrix theory for deep learning.
\newblock In {\em {NeurIPS}}, pages 2637--2646, 2017.

\bibitem[RR07]{RahimiRecht07}
Ali Rahimi and Benjamin Recht.
\newblock Random features for large-scale kernel machines.
\newblock In {\em {NeurIPS}}, pages 1177--1184. Curran Associates, Inc., 2007.

\bibitem[RR08]{RahimiRecht08}
Ali Rahimi and Benjamin Recht.
\newblock Weighted sums of random kitchen sinks: Replacing minimization with randomization in learning.
\newblock In {\em {NeurIPS}}, pages 1313--1320. Curran Associates, Inc., 2008.

\bibitem[RR17]{RudiRosasco17}
Alessandro Rudi and Lorenzo Rosasco.
\newblock Generalization properties of learning with random features.
\newblock In {\em {NeurIPS}}, pages 3215--3225, 2017.

\bibitem[Sch42]{Schoenberg42}
I.~J. Schoenberg.
\newblock Positive definite functions on spheres.
\newblock {\em Duke Mathematical Journal}, 9(1):96--108, 1942.

\bibitem[SH21]{ScetbonHarchaoui21}
Meyer Scetbon and Za{\"{\i}}d Harchaoui.
\newblock A spectral analysis of dot-product kernels.
\newblock In {\em {AISTATS}}, volume 130 of {\em Proceedings of Machine Learning Research}, pages 3394--3402. {PMLR}, 2021.

\bibitem[She99]{sheppard1899application}
William~Fleetwood Sheppard.
\newblock On the application of the theory of error to cases of normal distribution and normal correlation.
\newblock {\em Philosophical Transactions of the Royal Society of London. Series A, Containing Papers of a Mathematical or Physical Character}, 192:101--531, 1899.

\bibitem[Son26]{Song26}
Zhao Song.
\newblock Tight worst-case bounds for the smallest eigenvalue of relu {NTK} gram matrices.
\newblock {\em CoRR}, abs/2608.03368, 2026.

\bibitem[STR18]{SerraEtAl18}
Thiago Serra, Christian Tjandraatmadja, and Srikumar Ramalingam.
\newblock Bounding and counting linear regions of deep neural networks.
\newblock In {\em {ICML}}, volume~80 of {\em Proceedings of Machine Learning Research}, pages 4565--4573. {PMLR}, 2018.

\bibitem[Tro15]{Tropp15}
Joel~A. Tropp.
\newblock An introduction to matrix concentration inequalities.
\newblock {\em Found. Trends Mach. Learn.}, 8(1-2):1--230, 2015.

\bibitem[Ver20]{Vershynin20}
Roman Vershynin.
\newblock Memory capacity of neural networks with threshold and rectified linear unit activations.
\newblock {\em {SIAM} J. Math. Data Sci.}, 2(4):1004--1033, 2020.

\bibitem[Wil96]{Williams97}
Christopher K.~I. Williams.
\newblock Computing with infinite networks.
\newblock In {\em {NeurIPS}}, pages 295--301. {MIT} Press, 1996.

\bibitem[Zas75]{Zaslavsky75}
Thomas Zaslavsky.
\newblock {\em Facing Up to Arrangements: Face-Count Formulas for Partitions of Space by Hyperplanes}, volume~1 of {\em Memoirs of the American Mathematical Society}.
\newblock American Mathematical Society, 1975.

\bibitem[ZBH{\etalchar{+}}21]{ZhangEtAl17}
Chiyuan Zhang, Samy Bengio, Moritz Hardt, Benjamin Recht, and Oriol Vinyals.
\newblock Understanding deep learning (still) requires rethinking generalization.
\newblock {\em Commun. {ACM}}, 64(3):107--115, 2021.

\end{thebibliography}
